%% file: 00-main.tex
\documentclass{article}

\usepackage{microtype}
\usepackage{graphicx}
\usepackage{subcaption}
\usepackage{xltabular}

\usepackage{array}
\usepackage{ragged2e}
\usepackage{booktabs} % for professional tables

\usepackage{hyperref}
\usepackage{booktabs}
\usepackage{tabularx}

\usepackage[preprint]{icml2026}

\usepackage{amsmath}
\usepackage{amssymb}
\usepackage{mathtools}
\usepackage{amsthm}

\usepackage{enumitem}

\usepackage{multirow}
\usepackage{graphicx}
\usepackage{colortbl}
\usepackage{xcolor}
\usepackage{makecell} % 用于表头内的强制换行
\usepackage[capitalize,noabbrev]{cleveref}

\theoremstyle{plain}
\newtheorem{theorem}{Theorem}[section]
\newtheorem{proposition}[theorem]{Proposition}

\theoremstyle{definition}
\newtheorem{definition}[theorem]{Definition}

\theoremstyle{remark}

\usepackage[textsize=tiny]{todonotes}

\icmltitlerunning{$\mathtt{FibVLA}$: An Efficient Temporal Vision-Language-Action Model with Fibonacci Sampling
}

\begin{document}

\twocolumn[
  \icmltitle{
  % $\mathtt{FibVLA}$: Balancing Efficiency and Accuracy in  \\ Vision-Language-Action Models via Fibonacci Temporal Sampling 
  $\mathtt{FibVLA}$: An Efficient Temporal Vision-Language-Action Model \\with Fibonacci Sampling
    }

  % It is OKAY to include author information, even for blind submissions: the
  % style file will automatically remove it for you unless you've provided
  % the [accepted] option to the icml2026 package.

  % List of affiliations: The first argument should be a (short) identifier you
  % will use later to specify author affiliations Academic affiliations
  % should list Department, University, City, Region, Country Industry
  % affiliations should list Company, City, Region, Country

  % You can specify symbols, otherwise they are numbered in order. Ideally, you
  % should not use this facility. Affiliations will be numbered in order of
  % appearance and this is the preferred way.
  \icmlsetsymbol{equal}{*}
  \icmlsetsymbol{email}{$\ddagger$}

  \begin{icmlauthorlist}
    \icmlauthor{Li Lin}{seu}
    \icmlauthor{Wujun Xu}{seu,email}
    \icmlauthor{Weiwei Meng}{seu}
    \icmlauthor{Kaiwen Xia}{ntu,email}
    \icmlauthor{Kang Hao Cheong}{ntu}
    \icmlauthor{Shuai Wang}{seu}
  \end{icmlauthorlist}

  \icmlaffiliation{seu}{Southeast University, Nanjing, Jiangsu, China}
  \icmlaffiliation{ntu}{Nanyang Technological University, Singapore}

  % You may provide any keywords that you find helpful for describing your
  % paper; these are used to populate the "keywords" metadata in the PDF but
  % will not be shown in the document
  \icmlkeywords{Machine Learning, ICML}

  \vskip 0.3in
]

% this must go after the closing bracket ] following \twocolumn[ ...

% This command actually creates the footnote in the first column listing the
% affiliations and the copyright notice. The command takes one argument, which
% is text to display at the start of the footnote. The \icmlEqualContribution
% command is standard text for equal contribution. Remove it (just {}) if you
% do not need this facility.

% Use ONE of the following lines. DO NOT remove the command.
% If you have no special notice, KEEP empty braces:
\printAffiliationsAndNotice{\textsuperscript{$\ddagger$}Emails: 220242461@seu.edu.cn, kaiwen.xia@ntu.edu.sg. }  % no special notice (required even if empty)
% Or, if applicable, use the standard equal contribution text:
% \printAffiliationsAndNotice{\icmlEqualContribution}

\begin{abstract}
Vision-language-action models (VLAs), which leverage the cognition of multimodal information to infer physical-world actions, provide a generalized solution for embodied AI applications. Conventional VLAs usually concentrate on current digital cognition. While some efforts are made to enhance VLAs' reasoning capabilities by capturing temporal information, encoding the long-context history causes an efficiency-decreasing issue. To reconcile the conflict between capturing temporal information and maintaining inference efficiency in VLAs, this paper introduces $\mathtt{FibVLA}$, an efficient framework featuring temporal perception of long-context history. Specifically, we leverage logarithmic hindsight sampling to both proprioceptive states and visual frames to capture long-term temporal dependencies with minimal redundancy. For the action expert, we introduce the flow matching to produce action distributions, and the Fibonacci recurrent inference strategy to generate long-range planning steps based on real-time closed-loop feedback. Experiments demonstrate that $\mathtt{FibVLA}$ significantly improves action smoothness and success rates without retraining large-scale visual encoders. Efficiency analysis demonstrates superior real-time responsiveness compared to video-based baselines in real-world evaluations.
\end{abstract}

\input{01-Introdution}
\input{02-Related_works}
\input{03-Method}

\input{04-Experiment}

\section{Conclusion}
In this paper, we propose $\mathtt{FibVLA}$, an efficient temporal VLA framework designed to bridge the gap between temporal perception and real-time inference. 
We design a novel gradual-frequency sampling strategy based on Fibonacci sampling, which allows the model to directly reuse feature caches from the previous timestep during inference without incurring additional temporal encoding overhead, 
thereby capturing temporal context information for embodied tasks. Experimental results on three benchmarks and the real-world dataset show that $\mathtt{FibVLA}$ significantly improves performance in long-horizon tasks and demonstrates real-time robustness in the real world. 
In future work, we plan to leverage the Fibonacci recursive principle to explore replay buffer mechanisms, enhancing $\mathtt{FibVLA}$’s continual learning for embodied tasks under out-of-distribution scenarios.

% We devise a novel variable-frequency sampling paradigm that unifies logarithmic hindsight sampling with a channel-wise temporal encoder to capture essential manipulation dynamics while minimizing computational redundancy. Our results demonstrate that incorporating temporal context through our method significantly enhances success rates in long-horizon tasks compared to static baselines. Furthermore, the proposed Fibonacci recurrent inference strategy allows the model to achieve superior computational efficiency and robust generalization across diverse real-world scenarios.

% % Acknowledgements should only appear in the accepted version.
% \section*{Acknowledgements}

% \textbf{Do not} include acknowledgements in the initial version of the paper
% submitted for blind review.

% If a paper is accepted, the final camera-ready version can (and usually should)
% include acknowledgements.  Such acknowledgements should be placed at the end of
% the section, in an unnumbered section that does not count towards the paper
% page limit. Typically, this will include thanks to reviewers who gave useful
% comments, to colleagues who contributed to the ideas, and to funding agencies
% and corporate sponsors that provided financial support.

\section*{Impact Statement}

% Authors are \textbf{required} to include a statement of the potential broader
% impact of their work, including its ethical aspects and future societal
% consequences. This statement should be in an unnumbered section at the end of
% the paper (co-located with Acknowledgements -- the two may appear in either
% order, but both must be before References), and does not count toward the paper
% page limit. In many cases, where the ethical impacts and expected societal
% implications are those that are well established when advancing the field of
% Machine Learning, substantial discussion is not required, and a simple
% statement such as the following will suffice:

This paper advances the foundations of Embodied AI. By improving the efficiency of vision-language-action models, this work reduces the computational overhead for robotic training and inference, mitigating the environmental footprint of large-scale deployment. Furthermore, these efficiency gains facilitate the democratization of advanced robotic control on resource-constrained platforms. As fundamental research, its specific societal consequences primarily depend on downstream applications, and we foresee no specific negative impacts.

% The above statement can be used verbatim in such cases, but we encourage
% authors to think about whether there is content which does warrant further
% discussion, as this statement will be apparent if the paper is later flagged
% for ethics review.

% % In the unusual situation where you want a paper to appear in the
% % references without citing it in the main text, use \nocite
\nocite{langley00}

\bibliography{example_paper}
\bibliographystyle{icml2026}

%%%%%%%%%%%%%%%%%%%%%%%%%%%%%%%%%%%%%%%%%%%%%%%%%%%%%%%%%%%%%%%%%%%%%%%%%%%%%%%
%%%%%%%%%%%%%%%%%%%%%%%%%%%%%%%%%%%%%%%%%%%%%%%%%%%%%%%%%%%%%%%%%%%%%%%%%%%%%%%
% APPENDIX
%%%%%%%%%%%%%%%%%%%%%%%%%%%%%%%%%%%%%%%%%%%%%%%%%%%%%%%%%%%%%%%%%%%%%%%%%%%%%%%
%%%%%%%%%%%%%%%%%%%%%%%%%%%%%%%%%%%%%%%%%%%%%%%%%%%%%%%%%%%%%%%%%%%%%%%%%%%%%%%
\newpage
\appendix
\onecolumn

\section{Detailed Experimental Setups and Benchmarks}
\begin{figure*}[htbp] %
    \centering
    \includegraphics[width=1.0\textwidth]{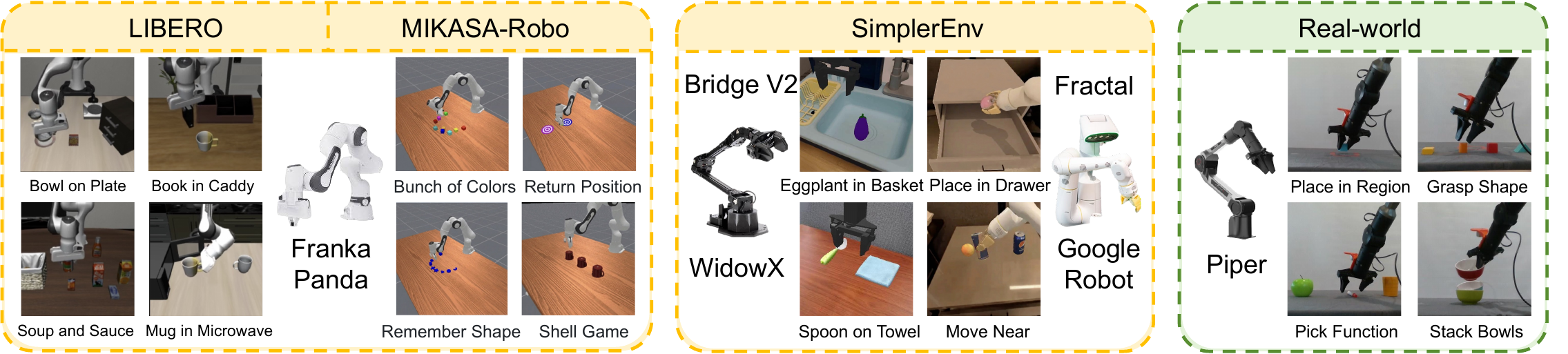} % 设置图片宽度
    \caption{Overview of experimental setups and task specifications. The evaluation encompasses three simulation benchmarks and real-world robotic tasks across diverse platforms, including the \textbf{Franka Panda} (LIBERO-Long and MIKASA-Robo), \textbf{WidowX 250} (SimplerEnv-Bridge), and \textbf{Google Robot} (SimplerEnv-Fractal), and the \textbf{Piper} manipulator (Real-world evaluation).}
    \label{fig:experiment1}
\end{figure*}
\input{table/Benchmark}
\section{Extended Description of Baselines}
\input{table/Baseline}
% You can have as much text here as you want. The main body must be at most $8$
% pages long. For the final version, one more page can be added. If you want, you
% can use an appendix like this one.

% The $\mathtt{\backslash onecolumn}$ command above can be kept in place if you
% prefer a one-column appendix, or can be removed if you prefer a two-column
% appendix.  Apart from this possible change, the style (font size, spacing,
% margins, page numbering, etc.) should be kept the same as the main body.
%%%%%%%%%%%%%%%%%%%%%%%%%%%%%%%%%%%%%%%%%%%%%%%%%%%%%%%%%%%%%%%%%%%%%%%%%%%%%%%
%%%%%%%%%%%%%%%%%%%%%%%%%%%%%%%%%%%%%%%%%%%%%%%%%%%%%%%%%%%%%%%%%%%%%%%%%%%%%%%
\section{Detailed Task Specifications}
\label{sec:subtasks}
\input{table/Instruction}
\newpage
\section{Per-Task Performance Analysis on LIBERO-Long}
\label{sec:libero-long}
\input{table/Libero-10}

\section{Scoring Rubric for Real-world Experiments}
\label{sec:rubric}
To comprehensively evaluate the error-recovery capabilities and execution fluency of our model in long-horizon tasks, we propose a quantitative scoring framework as follows:
\begin{itemize}
    \item \textbf{Task Decomposition}: Each long-horizon task is decomposed into five key sub-steps (i.e., action primitives), representing the critical stages of the manipulation sequence.
    \item \textbf{Scoring Criteria}: Each sub-step is evaluated based on execution quality using a tiered scoring system:
    \begin{itemize}
        \item \textbf{2 Points (Fluent Execution)}: The robot successfully completes the sub-step in a single, continuous motion without hesitation or significant deviation.
        \item \textbf{1 Point (Corrective Execution)}: The robot encounters minor deviations or stutters but successfully completes the sub-step through autonomous self-correction.
        \item \textbf{0 Point (Execution Failure)}: The robot fails to complete the sub-step or encounters an irreversible error, such as a collision or dropping the object.
    \end{itemize}
    \item \textbf{Statistical Protocol}: To ensure statistical significance, each task is conducted over 10 independent trials. The maximum possible score for a single task is calculated as: 
    \begin{equation}
        S_{max} = 5 \text{ sub-steps} \times 2 \text{ points} \times 10 \text{ trials} = 100
    \end{equation}
\end{itemize}
\input{table/Rubric}

\section{Hyperparameter Settings and Dataset Statistics}
\label{FibVLA_Hyperparameter}
\input{table/experimental_para}
\section{Real-world Experimental Setup and Task Execution}
\label{appendix: experiment}
The physical experimental platform is centered around the Agilex Piper robot arm, integrated with the Agilex Pika teleoperation kit for demonstration data collection. Visual feedback is provided by wrist-mounted and third-person Intel RealSense D435 cameras, supporting both manual teleoperation and autonomous model inference. To ensure the reproducibility of our experiments, we deliberately designed the environment using readily available components. The setup employs a standardized white background board and a black tablecloth to minimize environmental interference, while all task objects are common household items to facilitate the benchmarking of $\mathtt{FibVLA}$ in accessible settings.
\begin{figure*}[htbp] %
    \centering
    \includegraphics[width=0.72\textwidth]{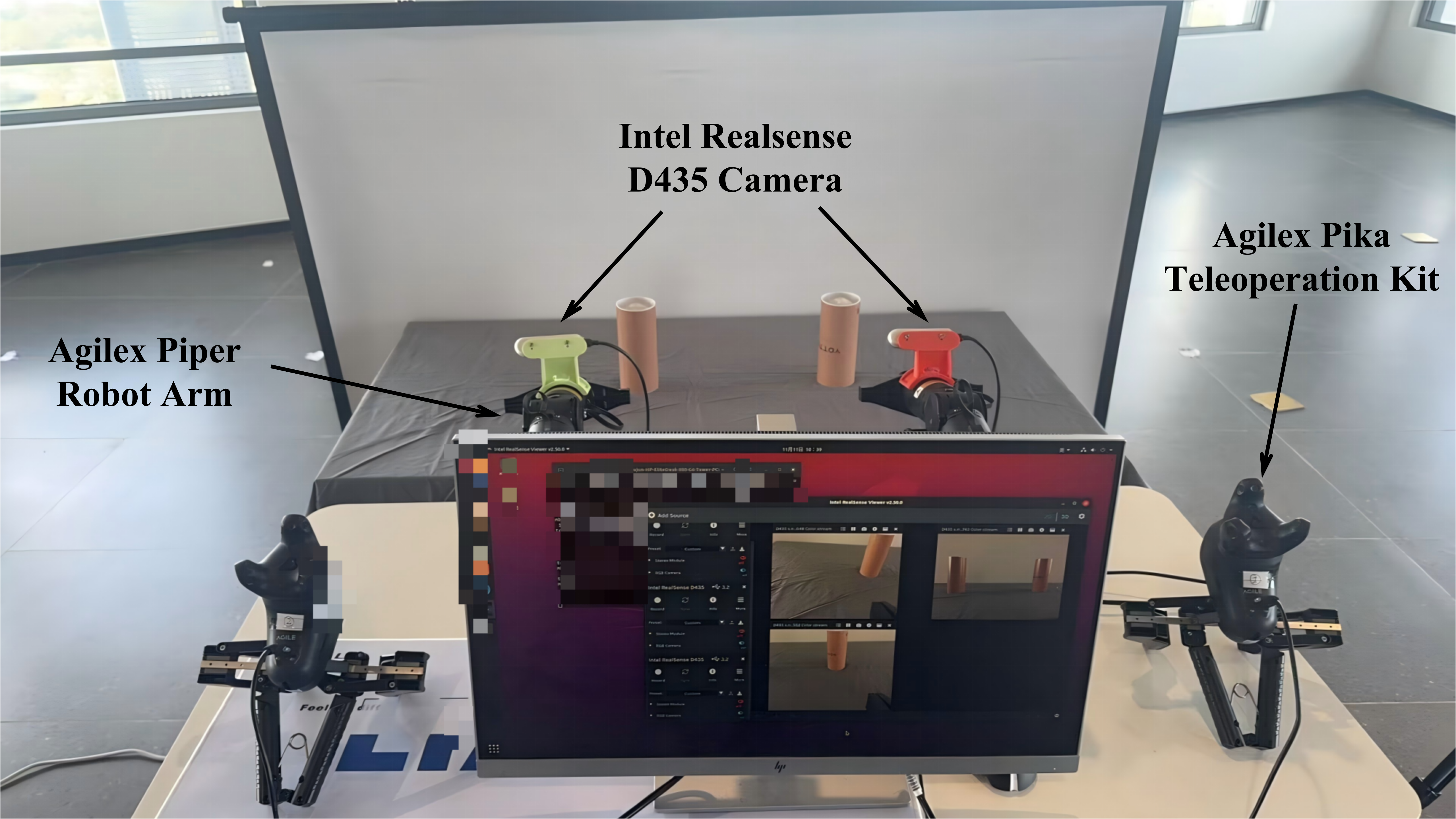} % 设置图片宽度
    \caption{Real-world experimental scenario}
    \label{fig:experiment2}
\end{figure*}

We use the ``Place Bowl (Test)" task as an illustrative example to demonstrate the real-world inference performance of $\mathtt{FibVLA}$. To enhance generalization, the demonstration data include balanced initial configurations (e.g., swapping the left-right positions of the bowl and plate as shown in Fig.~\ref{fig:step1}), and the model successfully executes the task regardless of the initial layout. Crucially, $\mathtt{FibVLA}$ exhibits a deep understanding of temporal logic rather than mere imitation learning. When external interventions revert the task state during inference (e.g., resetting from Fig.~\ref{fig:step3} back to Fig.~\ref{fig:step1}, or Fig.~\ref{fig:step5} back to Fig.~\ref{fig:step3}), the model autonomously perceives the environmental shift and re-executes the necessary preceding actions. Such dynamic error recovery proves that $\mathtt{FibVLA}$ explicitly models temporal information to achieve robust long-horizon reasoning.
\begin{figure*}[t] % 建议使用 [t] 确保跨栏图排在页顶，减少空白
    \centering
    % --- 第一张图：移向并夹取盘子 ---
    \begin{subfigure}[t]{0.19\textwidth}
        \centering
        \includegraphics[width=\linewidth]{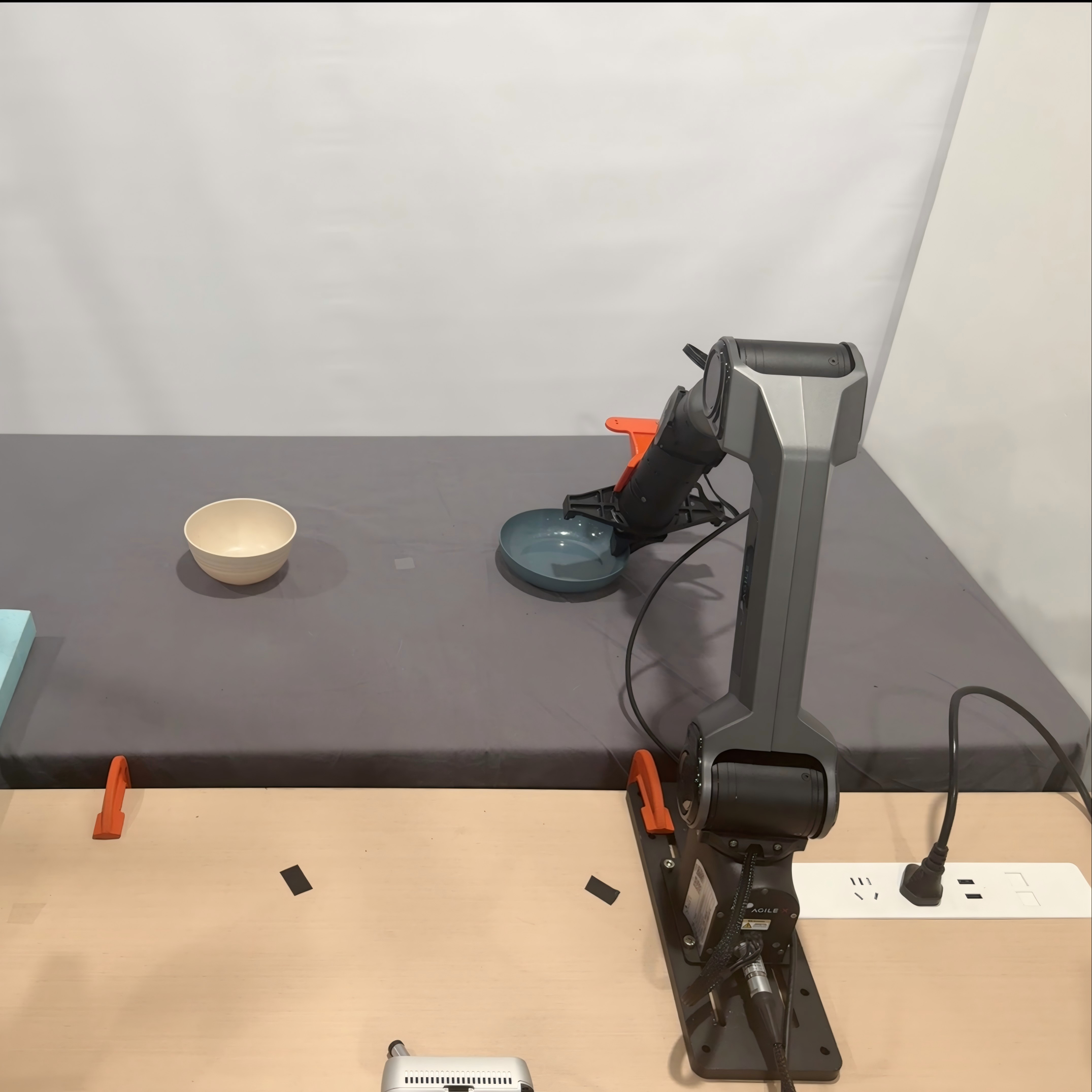}
        \caption{Approach and grasp plate}
        \label{fig:step1}
    \end{subfigure}
    \hfill
    % --- 第二张图：放置盘子 ---
    \begin{subfigure}[t]{0.19\textwidth}
        \centering
        \includegraphics[width=\linewidth]{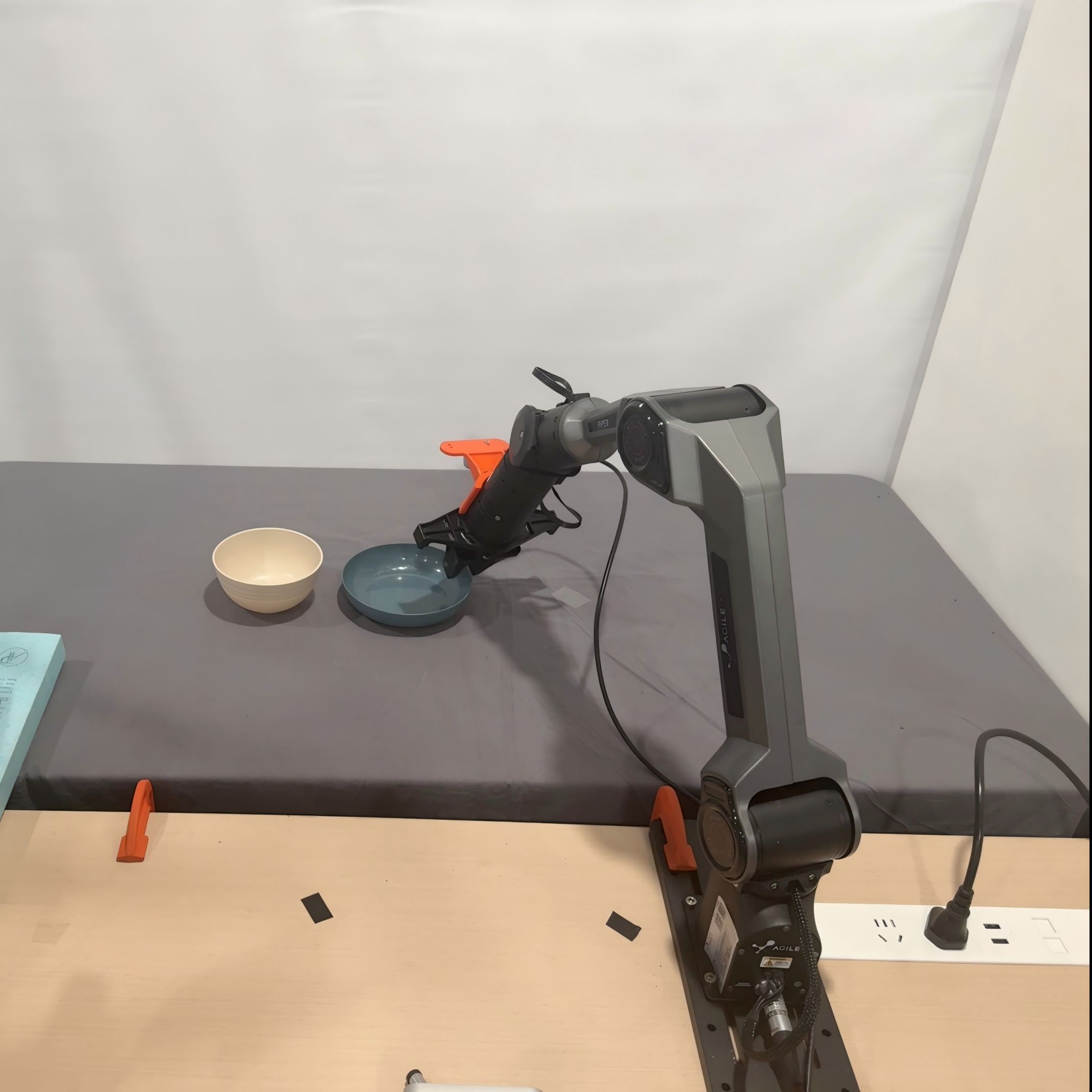}
        \caption{Place plate on center}
        \label{fig:step2}
    \end{subfigure}
    \hfill
    % --- 第三张图：移向并夹取碗 ---
    \begin{subfigure}[t]{0.19\textwidth}
        \centering
        \includegraphics[width=\linewidth]{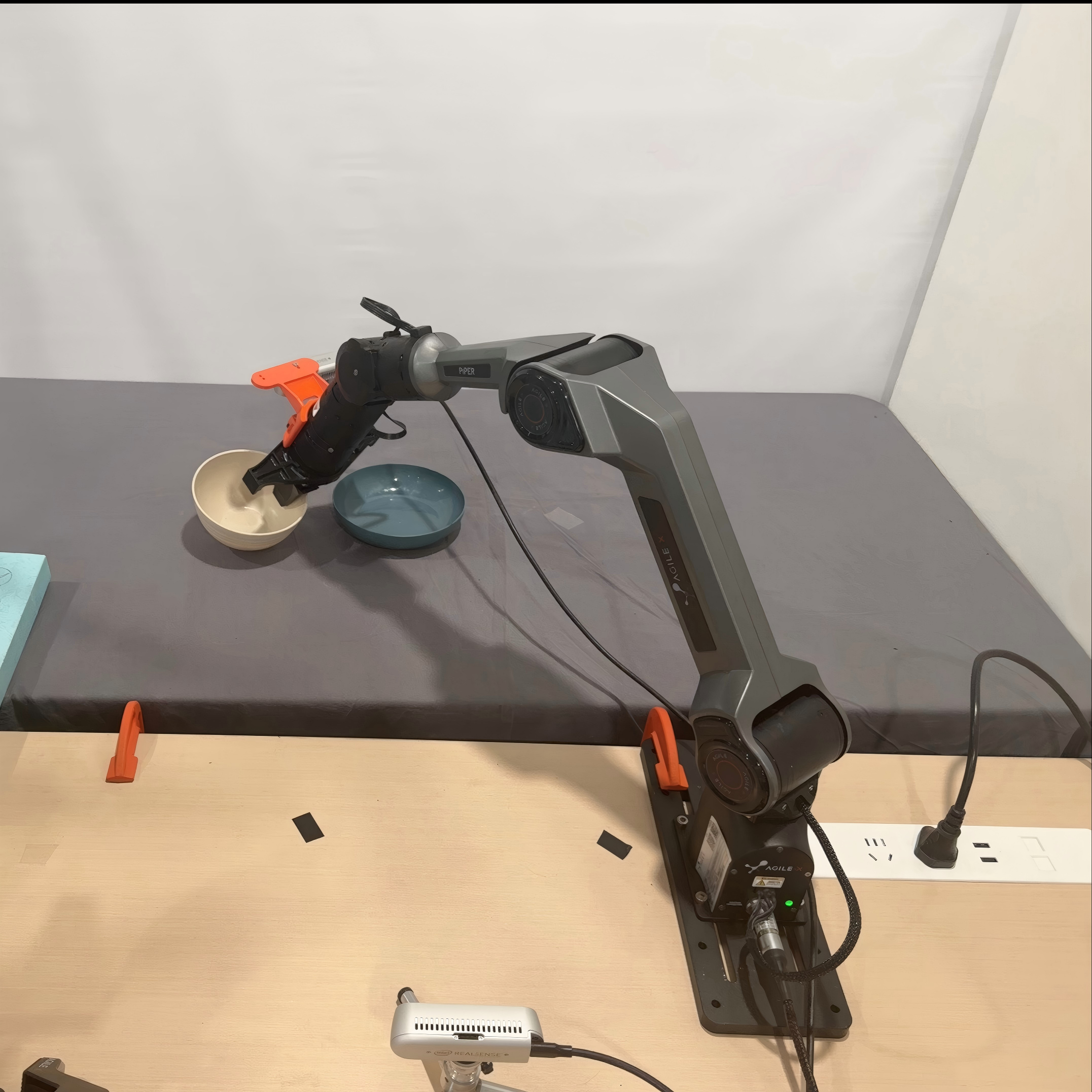}
        \caption{Approach and grasp bowl}
        \label{fig:step3}
    \end{subfigure}
    \hfill
    % --- 第四张图：将碗放在盘子上 ---
    \begin{subfigure}[t]{0.19\textwidth}
        \centering
        \includegraphics[width=\linewidth]{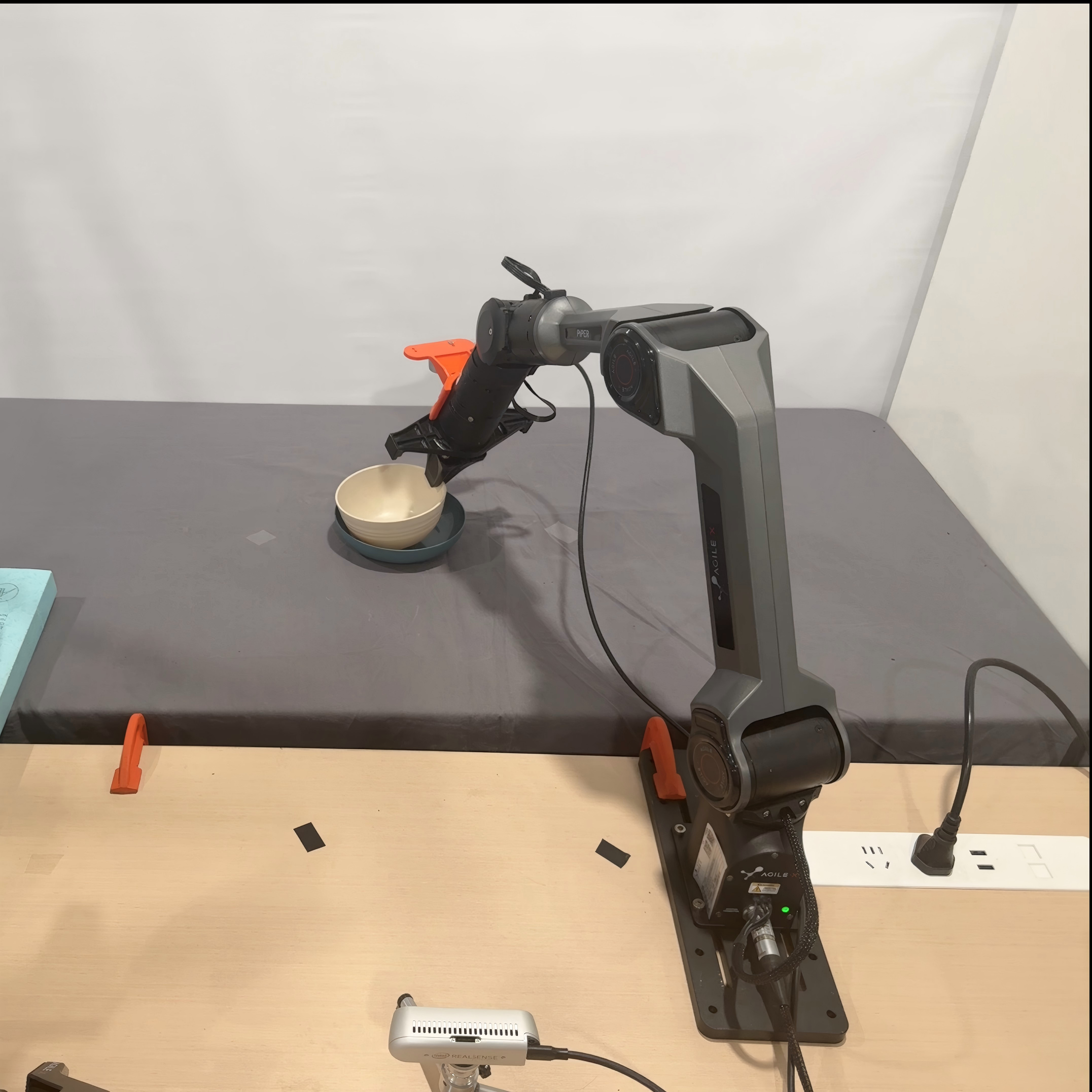}
        \caption{Place bowl on plate}
        \label{fig:step4}
    \end{subfigure}
    \hfill
    % --- 第五张图：回归原位 ---
    \begin{subfigure}[t]{0.19\textwidth}
        \centering
        \includegraphics[width=\linewidth]{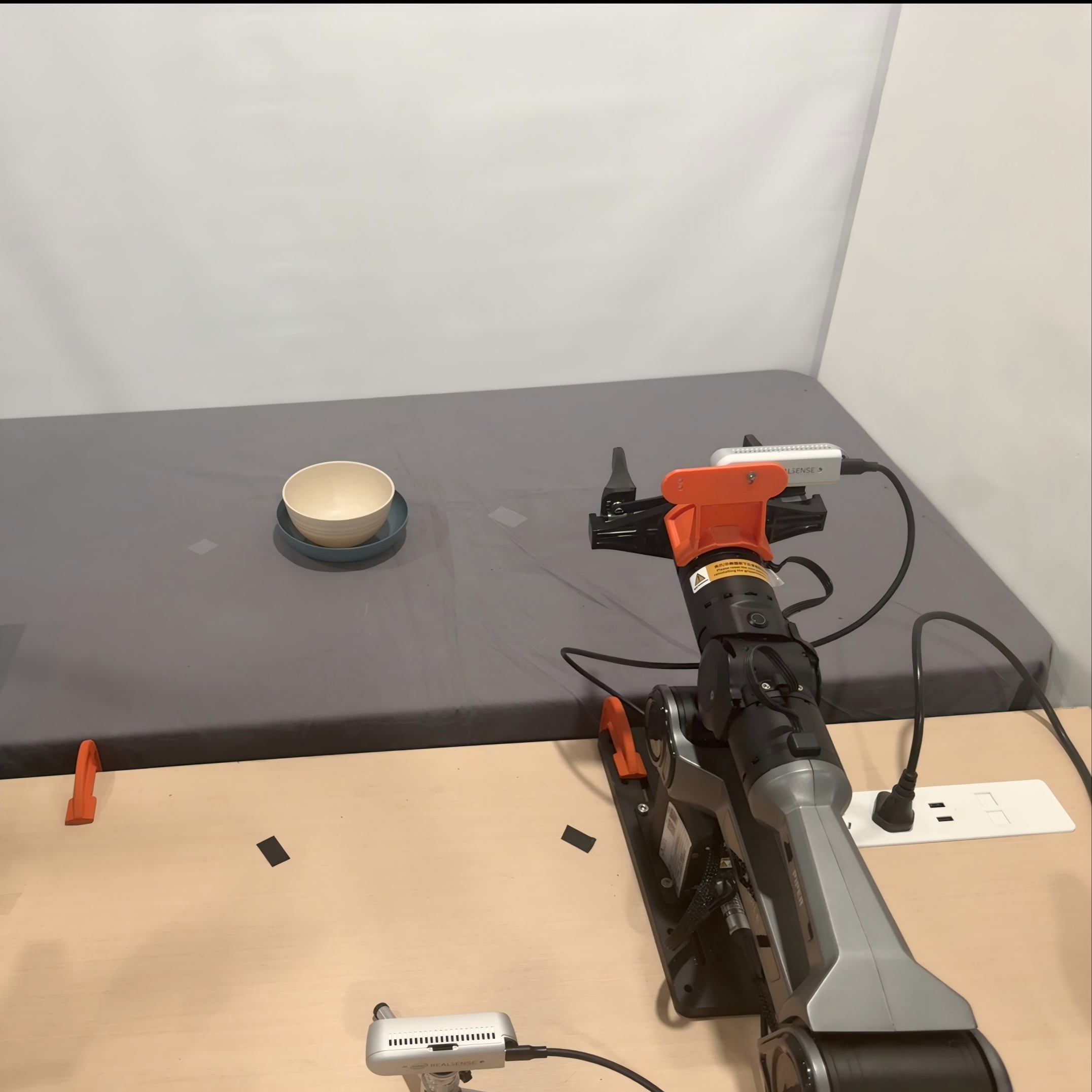}
        \caption{Return to home position}
        \label{fig:step5}
    \end{subfigure}

    \caption{Key-step snapshots of the "Place Bowl (Test)" task during real-world inference.}
    \label{fig:place_bowl_steps}
\end{figure*}

\newpage
\section{Proof of Upper Bound on Coincidence Count for Logarithmic Hindsight Sampling}
\label{proof}
\input{Proof}

\newpage
\section {Visualization of Channel-wise Temporal Encoding}
\label{mhi}
The image hindsight effectively represents the temporal evolution of the task within the latent space. A notable phenomenon in these visualizations is the variation in the morphology of the feature "trails." In the LIBERO environment, which operates at a lower sampling frequency of $10\text{Hz}$, the processed output exhibits elongated trails due to the larger inter-frame displacement. Conversely, in the real-world experiments conducted at $30\text{Hz}$, the encoding appears more compact and remains closely aligned with the object boundaries. Rather than being visual artifacts, these variations in trail length and distribution serve as a direct mapping of the object's motion dynamics, such as velocity and trajectory. This demonstrates that the CTE module successfully integrates the temporal progression of the task into the current visual representation.
\input{assets/HMI}

\end{document}

%% file: 01-Introdution.tex
\section{Introduction}
% The convergence of Vision-Language Models (VLMs)~\cite{li2022blip,driess2023palm,liu2023llava} and large-scale robotic datasets has established Vision-Language-Action models (VLAs) as a foundational framework for embodied agents. 
Vision-Language-Action models (VLAs) serve as a foundational paradigm for embodied agents by grounding language instructions into executable robotic control tokens~\cite{zitkovich2023rt,yu2025survey}.
Conventional VLAs~\cite{{black2024pi_0, liu2025hybridvla, kim2025openvla}} primarily rely on limited traceback observations, which are insufficient for long-horizon embodied tasks in which current actions depend causally on historical context~\cite{songhistory, lin2025hif}. 
Consequently, incorporating temporal context into VLA backbones is crucial to allow embodied agents, such as robotic manipulators, to achieve reliable performance in real-world executions.

Existing temporal modeling or inference in VLAs focuses on the reasoning capabilities of models given historical observations by prompting visual frames in VLM input context~\cite{zhengtracevla,mees2024octo,li2024towards,cheang2024gr} or predicting future subgoals to guide action generation~\cite{zhao2025cot,tian2024predictive,zhang2025up}. 
% To endow VLAs with temporal modeling capabilities for robotic manipulation, the straightforward approach is to directly input the entire sequence of historical frames. 
% However, this method leads to a token explosion, resulting in prohibitive computational costs that hinder the real-time control required for inference. 
Although they utilize the temporal information to enable VLAs understand the reasoning process of action generations, the encoding of history frames and state traces has to face a new efficiency problem. It is the token explosion issue, which results in prohibitive computational costs that hinder the real-time control required for inference.
Therefore, to balance efficiency and the utilization of historical context, recent research primarily falls into two categories: 
sparse representation learning~\cite{zhengtracevla, zhang20254d} and fixed-frequency sampling methods~\cite{shi2025memoryvla, lin2025hif}. 
% trajectory reconstruction strategies and 
While the former attempts to compress data through sparse representations, it relies on offline preprocessing, which limits the deployment flexibility of the model in robotic manipulation. 
In contrast, fixed-frequency sampling overlooks the non-uniform information density inherent in embodied tasks.
For example, the sampling frequency needed to understand task progress differs substantially from that needed to capture the robot’s proprioceptive state. 
% By aligning multimodal representations, VLA models enable the transformation of language instructions into executable robotic control tokens.
% To accommodate the complexity of long-horizon tasks in real-world physical environments, the traditionally instantaneous-observation-centered VLA frameworks, which often lack historical context, are naturally evolving toward temporal-aware architectures. 
% This evolution is crucial because effectively capturing these temporal dependencies 
% By effectively 
% However, a fundamental issue remains in extending these models beyond
% VLA models must incorporate long-horizon task planning that effectively captures temporal dependencies. 
% To this end, leveraging temporal context to model task progression is essential for embodied agents, such as robotic manipulators, 

\begin{figure}[!t]
  \begin{center}    \centerline{\includegraphics[width=\columnwidth]{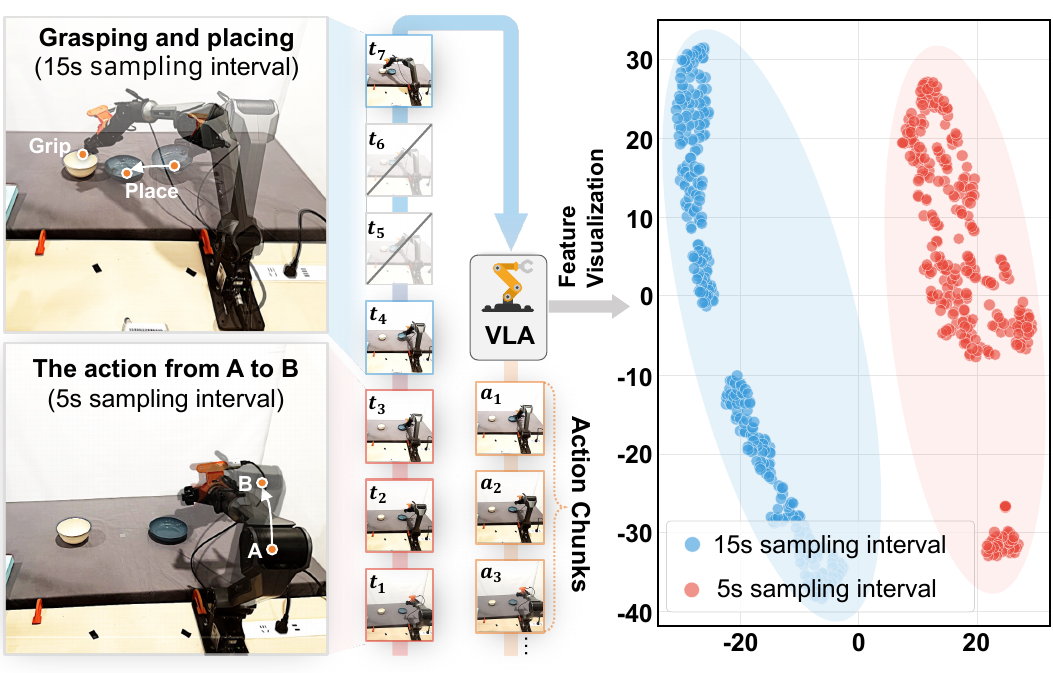}}
    \caption{Action chunks generated by VLA, and visualization of features extracted by VLA at different sampling intervals.}
    \label{fig:intro-figure}
  \end{center}
  \vspace{-0.6cm}
\end{figure}

%In this paper, we focus on \textit{how to achieve comprehensive temporal modeling with efficient inference to ensure precise robotic manipulation}.

% Pretrained Vision-Language Models (VLM) \cite{li2022blip,driess2023palm,liu2023llava} have established a unified framework for embodied intelligence by aligning linguistic and visual modalities, while the proliferation of open-source robotic datasets has further enabled data-driven control. Although Vision-Language-Action (VLA) models excel in cross-task semantic generalization, their dominant architectures remain constrained by an implicit Markov assumption, relying solely on instantaneous observations. This conflicts with the inherently Non-Markovian nature of complex robotic manipulation, where current states often derive from long-term causal accumulations. To mitigate this, existing approaches typically employ dense stacking of recent frames; however, we argue this addresses only the surface. By conflating physical inertia with true long-horizon episodic memory, this shortsighted strategy merely leverages short-term history to infer velocity and acceleration, ensuring kinematic smoothness and local action dependency. Without deep historical horizons, models lack effective representations of state evolution and event triggering. Consequently, reasoning degenerates into a passive response to current percepts, failing to track whether key states have changed. This lack of state traceability results in redundant actions and decision biases, ultimately undermining task-level structure and coherence in long-horizon tasks.

To demonstrate the non-uniform information density inherent within embodied tasks, we investigate the latent characteristics of robotic manipulation data. 
As shown in Figure~\ref{fig:intro-figure}, our empirical analysis reveals that representations derived from high and low sampling frequencies form distinct clusters within the feature space.
% exhibiting significant distributional divergence. 
This phenomenon suggests that high-frequency frames are primarily associated with fine-grained motion control, such as changes in proprioceptive states, whereas low-frequency frames correspond to the understanding of task context, such as the progression of subtasks. 
Based on the above observation, we argue that by integrating information from gradual-frequency sampling, the VLA model can enhance its comprehensive temporal perception, enabling the robot to execute high-frequency actions, while maintaining robust performance in long-horizon embodied tasks.
 
Based on these insights, this work aims to design an effective and efficient temporal VLA framework for robotic perception and inference, where there are two main challenges that need to be addressed:
First, \textit{how to efficiently encode sampled robotic frames on the perception side?} 
Despite gradual-frequency sampling, significant redundancy from static backgrounds and repeated features remains, requiring the model to extract key manipulation dynamics within a limited computational budget. 
Secondly, \textit{how to reduce the computational load of maintaining historical information during inference?}
Incorporating historical context at each step introduced additional retrieval and computation overhead, resulting in latency that undermines the real-time requirements of robotic control.

% While extending the temporal horizon is critical for capturing causal cues—such as verifying button activation or tracking object occupancy from seconds prior—efficiently processing long-horizon multimodal data remains a significant challenge. Naively stacking raw image frames to expand the memory window encounters irreducible bottlenecks. The linear accumulation of frames results in an explosion of visual tokens, which, coupled with the quadratic complexity of self-attention mechanisms, incurs prohibitive computational costs that severely constrain the achievable temporal context under limited memory resources. Long-horizon observations are typically dominated by static background pixels, creating a low signal-to-noise ratio where visual redundancy obscures sparse, task-relevant dynamics, making it difficult to distill decision-critical state transitions from the dense visual stream. The continuous maintenance of these long historical sequences during online inference imposes a heavy computational burden; this increased latency compromises the model's ability to meet the strict low-latency demands of high-frequency closed-loop control, inevitably leading to control lag and execution incoherence. 

% This status raises a critical research question: Can we design an efficient training and inference paradigm that retains long-horizon historical memory to capture physical dynamics while maintaining high-frequency real-time control responses, without introducing heavy computational burdens?

To deal with these challenges, we propose an efficient temporal \underline{\textbf{V}}ision-\underline{\textbf{L}}anguage-\underline{\textbf{A}}ction model with \underline{\textbf{Fib}}onacci sampling, called $\mathtt{FibVLA}$, which employs a flow matching generative strategy.
Specifically, $\mathtt{FibVLA}$ introduces logarithmic hindsight sampling at the input stage to effectively capture both high-frequency actions and low-frequency proprioception states. 
Furthermore, during the encoding stage, to efficiently represent the sampled historical frames, we design a channel-wise temporal encoding that leverages temporal features along the channel dimension to filter out irrelevant background in the robot’s historical frames, thereby reducing token noise.
For the action expert, we introduce the flow matching method to generate action distributions rather than discriminative action sequences.
Based on the action distributions, we design a Fibonacci recurrent inference strategy in the inference stage that exploits the recursive property of the Fibonacci sequence to align logarithmically sampled frames with preceding action chunks. 
% \linli{The term ``chunk" should be illustrated in Figure 1 or at least introduced in the content to describe the demonstration of data characteristics.}
Since the frames required for inference at the current timestamp are already included within the previous action chunk, avoiding redundant encoding of historical frames reduces inference overhead. Our contributions are summarized as follows:
\begin{itemize}[leftmargin=*]
    \item We propose $\mathtt{FibVLA}$, a framework that concentrates on the long context perception for efficient temporal modeling in VLAs. It develops logarithmic hindsight sampling with Fibonacci sampling to mathematically enable the direct reuse of prior action chunks during inference.
    \item We introduce channel-wise temporal encoding methods to capture background-invariant historical frames and a Fibonacci recurrent inference strategy aligned with preceding action chunks, substantially reducing computational overhead in temporal VLA inference. 
    \item We validate $\mathtt{FibVLA}$'s superior performance on four simulation benchmarks. Moreover, we collect a real-world dataset with more than 600k frames across 15 tasks, including single/dual-arm and long/short-horizon robotic manipulation, demonstrating the effectiveness and efficiency of $\mathtt{FibVLA}$ in the physical world.
\end{itemize}

%% file: 02-Related_works.tex
\section{Related Work}
\begin{figure*}[!ht] 
    \centering
    \includegraphics[width=0.95\textwidth]{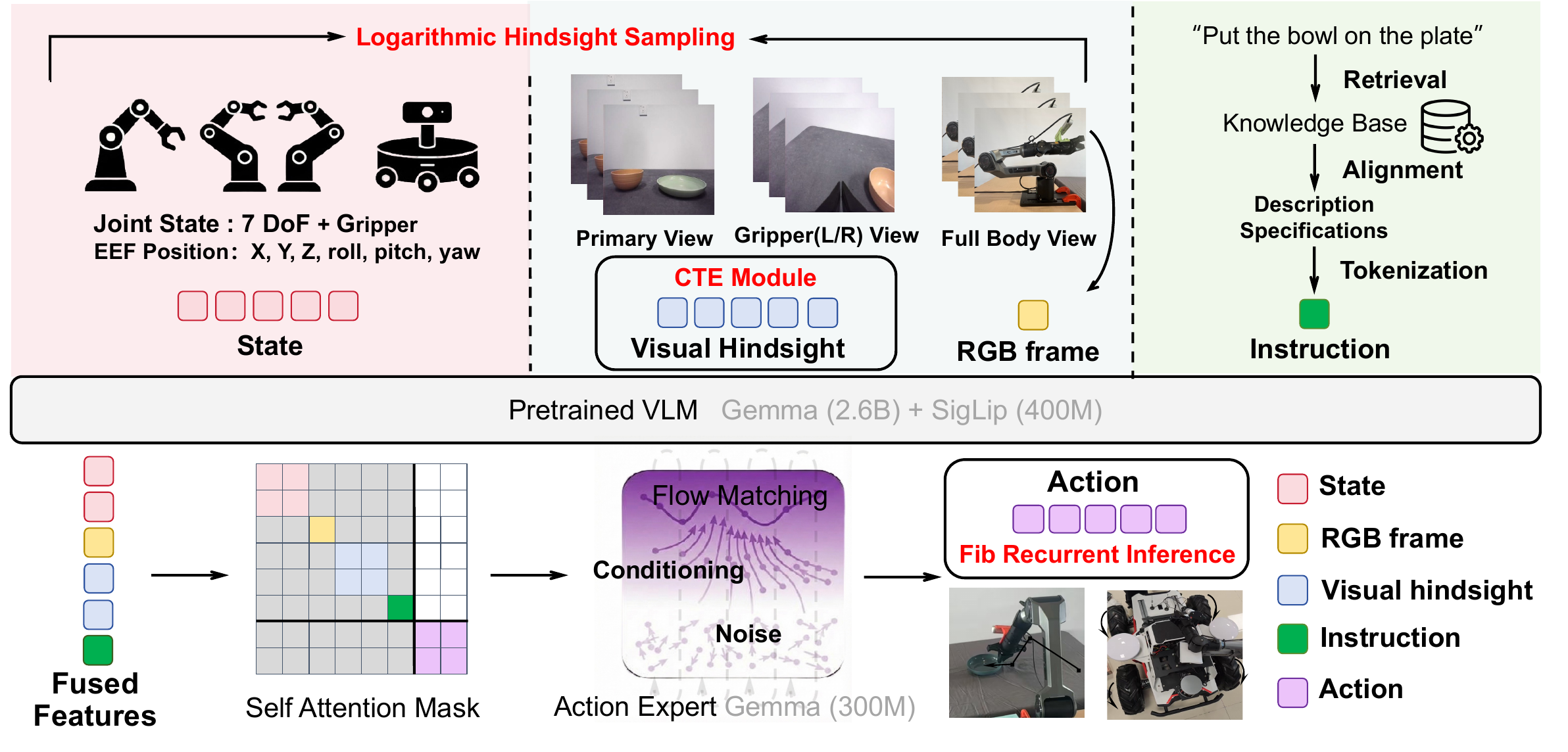} 
    \caption{The overview of $\mathtt{FibVLA}$, consisting of Logarithmic Hindsight Sampling, Channel-wise Temporal Encoding Module, and Fibonacci Recurrent Inference.
    % Our model unifies visual, language, and state modalities within a pretrained VLM backbone. It leverages {Logarithmic Hindsight Sampling} for temporal dependencies, processes visual features via a {Channel-wise Temporal Encoding Module}, and generates actions using {Fibonacci Recurrent Inference}.
    }
    \label{fig:arch}
    \vspace{-0.3cm}
\end{figure*}
\subsection{Vision-Language-Action Models}
% In the field of embodied AI, large language models have shown growing attractions and convincing capabilities by multimodal adaptions, which are referred to as vision-language models (VLMs) \cite{}. \xwj{\textbf{repetition,} 
To bridge the gap between the language inference of VLMs and the manipulation policies required for the intuitive human-robot interface, vision-language action models (VLAs)~\cite{yu2025survey,shao2025large,ma2024survey} are explored to provide a direct path for processing information from vision, language, and action modalities.
VLAs utilize visual foundation models~\cite{dosovitskiyimage,zhai2023sigmoid,radford2021learning} and LLMs to align vision and language embeddings, which are further trained on 
large-scale robot datasets~\cite{o2024openx,walke2023bridgedata,khazatsky2024droid} to perform language-conditioned tasks. 
Recent works achieve success using either autoregressive discretization (e.g., RT-1~\cite{brohan2023rt}, RT-2~\cite{zitkovich2023rt}, and OpenVLA~\cite{kim2025openvla}) or diffusion-based continuous denoising, such as $\pi_0$~\cite{black2024pi_0}, CogACT~\cite{li2024cogact}, and HybridVLA~\cite{liu2025hybridvla}.
However, these methods often operate under the implicit assumption of the Markovian state~\cite{jang2022bc-z,haarnoja2018soft,schulman2017ppo}, relying on instantaneous observations and overlooking the temporal dependencies arising from historical visual and proprioceptive states required for complex embodied tasks. 
For example, in tasks involving object occlusion, the scarcity of historical context prevents VLA models from recovering the last-known positions of target objects, thereby impeding successful task execution.

% Research on VLA models has evolved from foundational control paradigms to advanced generative architectures. Early works, notably RT-1 \cite{brohan2023rt} and RT-2 \cite{zitkovich2023rt}, established a general-purpose framework leveraging Transformer-based backbones and pre-trained VLMs, a direction further standardized by OpenVLA \cite{kim2025openvla}. Subsequent focus shifted to efficient generative models: Octo pioneered the integration of diffusion policies, while RDT-1B \cite{liurdt} adopted a Diffusion Transformer \cite{peebles2023dit} backbone with alternating condition injection for deep multimodal fusion. Addressing real-time control, $\pi_0$ \cite{black2410pi0} introduced a flow-matching-based "action expert" for high-frequency, non-autoregressive parallel action chunking, while LAPA \cite{yelapa} demonstrated the utility of latent action extraction from large-scale unlabeled video. Recent trends favor hierarchical architectures that balance high-level semantic understanding with precise low-level execution; exemplified by GR00T N1, these multi-stage designs integrate diverse generative mechanisms to enhance long-horizon task planning and execution, aiming to improve scalability and generalization.

\subsection{VLA Models Capturing Temporal Information}
To address partial observability in long-context robotic manipulation, VLAs have evolved from basic pixel-stacking to integrated historical contexts~\cite{li2023roboflamingo}. While these advancements enhance the capacity of VLA models to capture temporal dependencies, they are increasingly bottlenecked by both the token explosion and the prohibitive computational overhead associated with long-horizon sequences~\cite{jiang2022vima,dao2022flashattention}. 
Although representative methods, such as TraceVLA~\cite{zhengtracevla} and 4D-VLA~\cite{zhang20254d}, employ trajectory or 4D reconstructions for data reduction, their reliance on intensive offline preprocessing severely compromises real-time responsiveness and deployment flexibility. 
Furthermore, existing memory-based~\cite{shi2025memoryvla} and predictive strategies~\cite{lin2025hif} utilize fixed-frequency sampling, overlooking the temporal non-uniformity of robotic manipulation.
In practice, critical signals are often sparse and concentrated within irregular temporal windows~\cite{james2022coarse,mandlekar2022matters}, yet current models expend resources on redundant frames, thereby risking the loss of pivotal transient information~\cite{feichtenhofer2019slowfast}. 
Thus, adaptively extracting essential historical information while maintaining inference efficiency and accuracy remains a fundamental issue in VLA temporal modeling.

%% file: 03-Method.tex
\section{Methodology}

\subsection{ Problem Formulation and Overview}

\subsubsection{Problem Formulation}
We formulate the robotic manipulation task as a sequential decision-making process. At current timestamp $t$, the system receives a natural language instruction $\mathcal{L}$, a current visual observation $o_t \in \mathbb{R}^{H \times W \times 3}$ with the resolution of $H \times W$, and a proprioceptive state $s_t \in \mathbb{R}^{d_s}$. The objective is to predict a future action chunk $\mathcal{A}_{t:t+L}=\{a_t, a_{t+1}, \dots, a_{t+L}\}$, where $L$ denotes the action chunk length and each action $a_t \in \mathbb{R}^{d_a}$. The policy $\pi$ relies solely on the instantaneous observation, modeled as:
\begin{equation}
    \mathcal{A}_{t:t+L} \sim \pi(\cdot \mid \mathcal{L}, o_t, s_t)
\end{equation}
To capture temporal dynamics, a standard approach extends this by conditioning on a continuous history window. Let $T$ represent the lookback window size. We define the continuous visual history as $\mathcal{O}_T = \{o_t, o_{t-1}, \dots, o_{t-T}\}$ (where $t > T$) and the corresponding state history as $\mathcal{S}_T$. The policy is further formalized as:
\begin{equation}
    \mathcal{A}_{t:t+L} \sim \pi(\cdot \mid \mathcal{L}, \mathcal{O}_T, \mathcal{S}_T)
\end{equation}

\subsubsection{Overview}
In this section, we propose the $\mathtt{FibVLA}$ framework based on flow matching generation to achieve efficient temporal modeling. 
At the input stage of $\mathtt{FibVLA}$,
% While the continuous history window $\mathcal{O}_T$ incorporates temporal information, it often introduces redundancy. To address this, we propose a method based on a non-continuous, sparse history sequence.
we introduce a monotonically increasing sparse temporal index set $\mathcal{K} = \{k_1, \dots, k_N\}$, containing $N$ sampled frames, which is used to construct the observation set $\mathcal{O}_{\mathcal{K}}$ and the corresponding state set $\mathcal{S}_{\mathcal{K}}$. 
% \begin{equation}
%     \mathcal{O}_{\mathcal{K}} = \{o_{t-k} \mid k \in \mathcal{K}, t-k \ge 0\}.
% \end{equation}
Our proposed $\mathtt{FibVLA}$, parameterized by $\theta$, generates the action chunk by conditioning on these non-uniformly sampled frames as $\mathcal{A}_{t:t+L}= \pi_\theta(\cdot \mid \mathcal{L}, \mathcal{O}_{\mathcal{K}}, \mathcal{S}_{\mathcal{K}})$. 
The overall architecture of the $\mathtt{FibVLA}$ framework is shown in Figure~\ref{fig:arch}. This framework unifies the processing of multimodal inputs, including language instructions, multi-view visual frames, and state information. All inputs are first projected into a unified token sequence and then fed into the PaliGemma backbone~\cite{beyer2024paligemma}. Specifically, the input first leverages logarithmic hindsight sampling to capture temporal dependencies in both visual frames and state, while the sampled visual frames are further processed by a channel-wise temporal encoding (CTE) Module. Then, the encoded multimodal features, guided by a customized prefix attention mask, are employed to generate the action distributions through a flow-matching-based action expert. Finally, the robotic action is precisely generated using the Fibonacci recurrent inference strategy. 

\subsection{Logarithmic Hindsight Sampling}
\label{LHS}
Based on the analysis in Figure~\ref{fig:intro-figure}, the VLA model enhances full-scale temporal perception in embodied tasks by integrating gradual-frequency sampling while balancing computational overhead.
% To mitigate temporal myopia, models require access to longer historical contexts. However, using long continuous windows is computationally prohibitive, making sparse and discrete temporal sampling necessary. 
A common approach is Long-Short Term Hybrid Sampling~\cite{wang2016temporal,lin2019tsm}, which maintains multiple temporal buffers to densely sample distant past states with a fixed stride. Given that distant states change slowly and provide limited information gain, this method is inefficient in allocating the sampling budget.
% recent history while uniformly sampling 
% Although this strategy reduces computation to some extent, it suffers from notable limitations: it requires maintaining multiple temporal buffers, increasing system complexity, and it allocates sampling budget inefficiently, as distant states often evolve slowly and yield diminishing information gains. 
Ideally, the sampling density should decay naturally with temporal distance, following a dense-near, sparse-far pattern. Probabilistic Quantile Sampling~\cite{li2019enhancing,zhou2021informer} partially follows this intuition, but its sampling points do not align with discrete time steps and rely on distributional assumptions and quantile computations, leading to inconsistencies between training and inference.

\begin{figure}[!t]
  \begin{center}    \centerline{\includegraphics[width=\columnwidth]{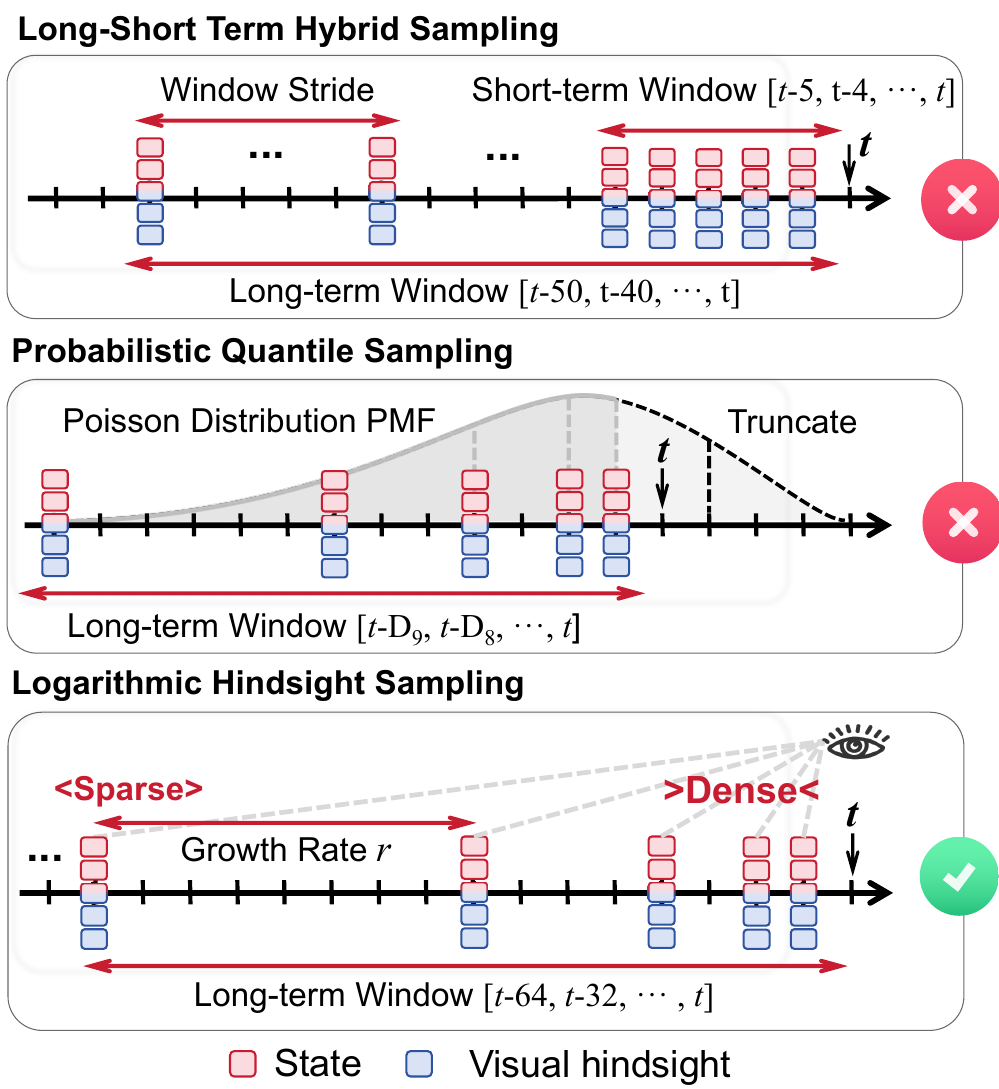}}
    \caption{Comparison among temporal sampling strategies.}
    \label{fig:sample}
  \end{center}
  \vspace{-1cm}
\end{figure}

To address the above issue, we propose Logarithmic Hindsight Sampling, which constructs a discrete set of temporal indices $\mathcal{K}$ on a logarithmic scale. This approach preserves high-frequency information for recent time steps while efficiently covering long-term history. Formally, given a minimum sampling interval $q_{min}$ and a growth rate $r > 1$, the sampling points are defined as:
\begin{equation}
k_i = \lfloor q_{min} \cdot r^i \rfloor
\label{eq:5}
\end{equation}
To mitigate index collisions caused by discretization artifacts (i.e., the floor operation), we impose a recursive sparsity constraint on the sequence:
\begin{equation}
k_i \ge k_{i-1} + k_{i-2}, \quad \forall i > 2
\label{eq:6}
\end{equation}
This constraint implicitly establishes a theoretical lower bound for the growth rate $r$, ensuring that the sampling sequence remains strictly monotonic after discretization and inherently eliminating redundancy. In practice, the growth rate can be flexibly adjusted according to the control frequency to balance long-term coverage with the precision required to capture critical state transitions.
% Motivated by these observations, we propose Logarithmic Hindsight Sampling. This method constructs a discrete and unique set of temporal indices on a logarithmic scale, preserving high-frequency for recent time steps while efficiently covering long-term history. Formally, given a minimum sampling interval $q_{\text{min}}$ and a growth rate $r>1$, the index set is defined as
% \begin{equation}
% \mathcal{K} = \operatorname{unique}\!\left(\left\lfloor q_{\text{min}} \cdot r^i \right\rfloor\right),
% \end{equation}
% where the floor operation aligns sampling points with discrete time steps and $\operatorname{unique}(\cdot)$ removes duplicates.

% To analyze information redundancy, we abstract the sorted index sequence $S=\{x_0, x_1, \dots\}$ with the following recursive sparsity constraint:
% \begin{equation}
% x_i \ge x_{i-1} + x_{i-2}, \quad \forall i \ge 2.
% \end{equation}
% This condition guarantees that the effective growth rate asymptotically exceeds the golden ratio $\phi \approx 1.618$, theoretically preventing integer collisions and rendering deduplication unnecessary in the ideal case. In practice, the growth rate can be adjusted according to the control frequency to avoid missing critical state transitions. The Fibonacci sequence can be viewed as a special-case approximation of logarithmic sampling. When combined with parallel action chunk generation, this design further improves inference efficiency; implementation details are discussed in Section~\ref{subsec:fib_inference}.

\subsection{Channel-wise Temporal Encoding}
\label{subsec:mhi_generation}
\begin{figure*}[t] %
    \centering
    \includegraphics[width=0.95\textwidth]{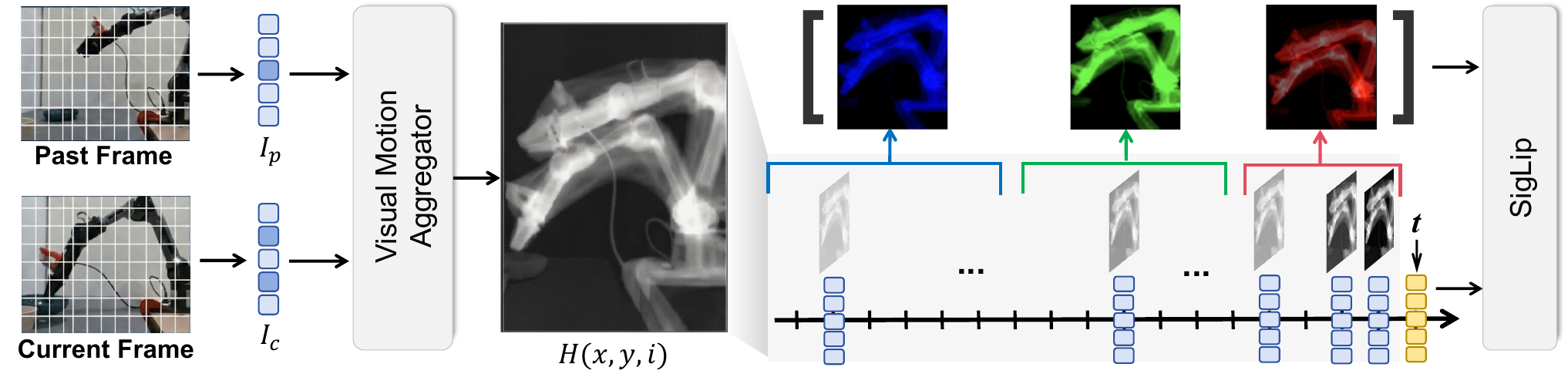} % 设置图片宽度
    \caption{Illustration of the channel-wise temporal encoding module. Specifically, $H(x, y, i)$ leverages spatial discrepancies between current and past frames to identify movement, 
    % enhancing dynamic regions and filtering out static backgrounds. 
    Historical frames are then integrated into the composite input via channel-wise mapping.
    }
    \label{fig:channel-wise}
    \vspace{-0.2cm}
\end{figure*}

While optimized sampling expands the temporal receptive field, scaling it to high-dimensional visual inputs remains challenging: long-horizon visual frames are often dominated by redundant backgrounds that mask critical task dynamics~\cite{tong2022videomae}, and Transformer self-attention incurs a computational overhead that scales quadratically with sequence length~\cite{dao2022flashattention}. To overcome these issues, we introduce a channel-wise temporal encoding module to compress discrete temporal dynamics into a compact representation. Specifically, the CTE module encodes temporal dynamics by using the visual motion aggregator to process sampled sparse frames $\mathcal{O}_{\mathcal{K}}$ based on the temporal index set $\mathcal{K}$. 
First, the motion difference $D(x,y,i)$ for each $i \in [1, N-1]$ is derived via frame differencing on the sampled visual frames $I(x,y,t)$:
\begin{equation}
D(\cdot,i) = |I(\cdot, t - k_i) - I(\cdot, t - k_{i+1})|
\label{eq:diff_image}
\end{equation}
Next, a binary motion mask $\Psi(x,y,i)$ is obtained using a predefined threshold $\xi$:
\begin{equation}
\Psi(\cdot,i) =
\begin{cases}
1, & \text{if } D(\cdot,i) > \xi \\
0, & \text{otherwise}
\end{cases}
\label{eq:mhi_binary}
\end{equation}
Then, the temporal encoding frame $H(x,y,i)$ is generated through the following recursive logic:
\begin{equation}
H(\cdot,i) =
\begin{cases}
\tau, &\text{if } \Psi(\cdot,i) = 1 \\
\max(0, H(\cdot,i+1) - \delta),  &\text{otherwise}
\end{cases}
\label{eq:mhi_update}
\end{equation}
where $\tau$ denotes the maximum intensity duration, and $\delta$ is a decay parameter. This recursive process proceeds in the direction of decreasing temporal lag. The final output, $H(x, y, 1)$, serves as the processed motion history image~\cite{ahad2012motion}. In this representation, brighter pixels correspond to recent motion, while darker pixels retain earlier motion traces, thereby forming a ``visual trail'' that provides explicit temporal cues.

To fully leverage multi-scale temporal features obtained through gradual-frequency sampling, the CTE module divides the visual history into three temporal ranges: Near, Mid, and Far. As shown in Figure~\ref{fig:channel-wise}, features from these ranges are respectively mapped to the R, G, and B channels of the PaliGemma visual encoder (SigLip), forming the hindsight feature $\hat{o}_t$ (see Appendix~\ref{mhi} for examples).  
Finally, the current RGB frame $o_t$ is retained as a semantic anchor and combined with $\hat{o}_t$ as input to the visual encoder, thereby integrating the embodied task’s temporal information with the scene’s color and texture details.

\subsection{Fibonacci Recurrent Inference}
\label{subsec:fib_inference}
During the VLA inference stage, the sampled $\mathcal{K}$ based on Eq.(\ref{eq:6}) causes misalignment between the frames used to infer the next action chunk and the historical frames when $k_i>k_{i-1}+k_{i-2}$, as shown in Figure~\ref{fig:lazy_inference}.
% when considering temporal information, the normal logarithmic sampling method results in the frames required for inferring the next action chunk being misaligned with the historical frames. 
This misalignment prevents the direct reuse of historical feature tokens, forcing the model to re-encode temporal information at every step and thereby limiting inference speed. 
In contrast, when $k_i=k_{i-1}+k_{i-2}$, the sampling strategy perfectly aligns with the additive recursive property inherent to the Fibonacci sequence, allowing historical feature tokens to be reused during the inference of the next action chunk.  
% To maximize the reuse of temporal encoding information for each historical frame during inference, 
Thus, we propose a Fibonacci recurrent inference strategy, leveraging the inherent additive recursive property of the Fibonacci sequence. We prove that this strategy provides the unique analytical solution for maximizing historical information reuse under sparse sampling constraints (proof detailed in Appendix~\ref{proof}).
% To further compress closed-loop inference latency while maintaining long-horizon perception, we leverage the inherent additive recurrence property of the Fibonacci sequence to propose the Fibonacci Recurrent Inference Mechanism. Building upon the sparse sampling set $\mathcal{K}$ defined previously (Eq.~\ref{eq:5} and Eq.~\ref{eq:6}), we emphasize that the Fibonacci sequence is not merely a logarithmic sampling strategy adhering to the ``dense-near, sparse-far'' principle, but fundamentally stands as the unique optimal solution for maximizing the reuse of historical information under sparse sampling constraints (proof provided in Appendix~\ref{proof}).
\begin{figure}[!t]
  \begin{center}
  \centerline{\includegraphics[width=\columnwidth]{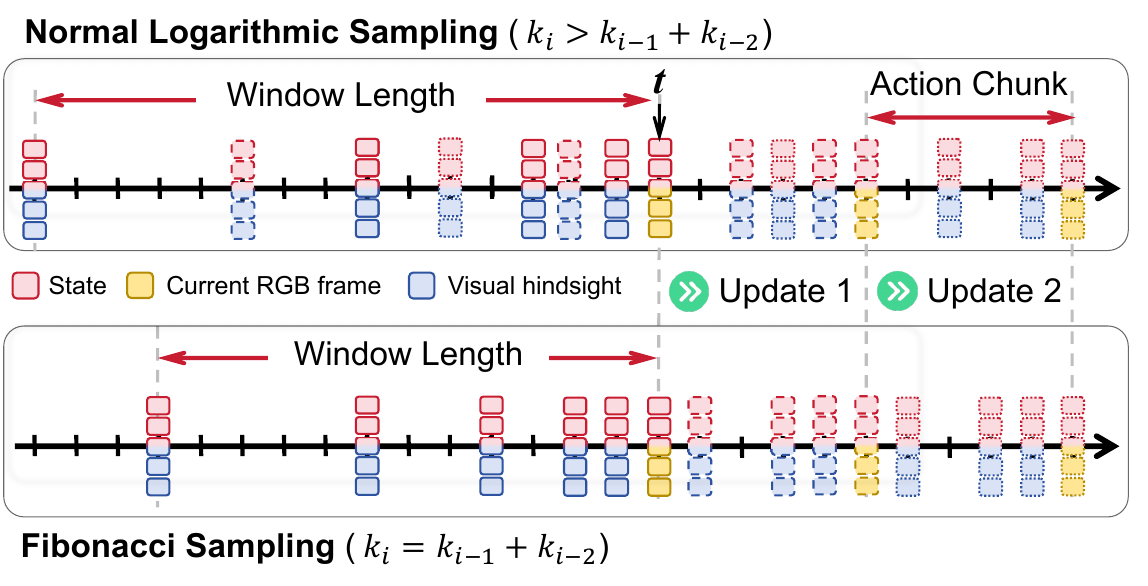}}
    \caption{Schematic of Fibonacci Recurrent Inference.}
    \label{fig:lazy_inference}
  \end{center}
  \vspace{-0.8cm}
\end{figure}

% To fully exploit this theoretical property and avoid redundant sampling, we seek a strict alignment between the action chunk length $L$ and the sparse sampling sequence $\mathcal{K}$. 
Let the system receive visual input $\mathcal{O}_{\mathcal{K}}^t$ at time $t$ and execute an action sequence of length $L$, advancing the system time to $t' = t + L$. Utilizing the Fibonacci property $k_i = k_{i-1} + k_{i-2}$, and setting the action update step to $L = k_{i-2}$, the historical frames at time $t$ and the sampling points at time $t'$ satisfy the following equality:
\begin{equation}
    (t + k_{i-2}) - k_i = t - (k_i - k_{i-2}) = t - k_{i-1}
    \label{eq:fib_alignment}
\end{equation}
Eq.(\ref{eq:fib_alignment}) indicates that by setting the update step (i.e., action chunk length) to $L = k_{i-2}$, the historical frame located $k_{i-1}$ steps before time $t$ corresponds to the frame located  $k_i$ before time $t'$. 
% Consequently, short-term history frames from the previous timestamp evolve into long-term history frames at the current timestamp. 
Furthermore, for sampling point $k_{i-3}$ required at $t'$, its effective timestamp satisfies $t' - k_{i-3} = t + (k_{i-2} - k_{i-3}) > t$, placing it strictly after the previous sampling time $t$. This confirms that all high-frequency updates occur strictly within the newly elapsed window $[t, t+L]$, thereby preserving the integrity of the historical structure without any misaligned sampling points.

Based on this derivation, we dynamically bind the action chunk length to specific Fibonacci terms, enabling the feature tokens in the KV cache to be precisely aligned with the sampling points at the next time step. As the sampling depth increases, the ratio between adjacent sampling intervals converges to the golden ratio, naturally maintaining a logarithmic sampling distribution.

%% file: 04-Experiment.tex
\section{Experiment}
\subsection{Experiment Setups}
We have conducted a comprehensive series of experiments to evaluate the effectiveness of $\mathtt{FibVLA}$, including:  
1) Three simulated benchmark tests to measure its manipulation performance across diverse tasks.
2) Robotic deployment experiments in real-world environments.
3) Ablation studies to verify the contribution of each module. 
4) Efficiency analysis to assess the model’s inference performance.
% We evaluate the effectiveness of $\mathtt{FibVLA}$ through an extensive suite of experiments. Our evaluation begins with multiple simulation benchmarks to measure fundamental manipulation performance across diverse tasks. To verify the model's cross-domain adaptability, we further conduct real-world robotic deployments in unstructured environments. Additionally, we perform comparative timing analyses to assess inference efficiency under various sampling strategies. Finally, detailed ablation studies are provided to isolate and analyze the individual contributions of each core module.

\noindent \textbf{Benchmarks \& Datasets.} 
We employ comprehensive benchmarks to evaluate task generalization and long-horizon manipulation capabilities. 1) Diverse simulation benchmarks, including \textbf{LIBERO}~\cite{liu2023libero} and \textbf{MIKASA-Robo}~\cite{cherepanov2025mikasa}, alongside large-scale policy evaluations on the \textbf{Bridge}~\cite{walke2023bridgedata} and \textbf{Fractal}~\cite{brohan2023rt} via SimplerEnv~\cite{li2025simpler}. 
2) \textbf{Real-world dataset}, which is collected from a physical platform built on the Piper robotic arm, comprising over 600k frames of data across 15 tasks, including single/dual-arm and long/short-horizon robotic manipulation. Further details on the specific configurations for these benchmarks and datasets are available in Appendix~\ref{sec:appendix_benchmarks} and Appendix~\ref{sec:rubric}.

\noindent \textbf{Baselines.} To verify the effectiveness of $\mathtt{FibVLA}$, we compare it against a range of representative state-of-the-art VLA models. These baselines are primarily categorized into three types: 1) Classic VLM-based policies, including 
\textbf{RT-1-X}~\cite{brohan2023rt}, 
\textbf{RT-2-X}~\cite{zitkovich2023rt}, \textbf{OpenVLA}~\cite{kim2025openvla}, {and \textbf{OpenVLA-OFT}~\cite{kim2025fine}}. 2) Generative policies based on diffusion or flow matching, such as \textbf{Octo}~\cite{ghosh2024octo} and 
$\mathbf{\pi_0}$~\cite{black2024pi_0}. 3) VLA variants focusing on enhanced spatio-temporal reasoning and cognitive planning, encompassing 
\textbf{CogACT}~\cite{li2024cogact}, 
\textbf{TraceVLA}~\cite{zhengtracevla}, \textbf{SpatialVLA}~\cite{qu2025spatialvla}, 
\textbf{4D-VLA}~\cite{qu2025spatialvla}, and 
\textbf{HiF-VLA}~\cite{lin2025hif}. For detailed descriptions of all baselines and the implementation details of $\mathtt{FibVLA}$, please refer to Appendix~\ref{sec:appendix_baselines} and Appendix~\ref{FibVLA_Hyperparameter}, respectively. 
% We compare FibVLA against representative state-of-the-art robotic learning models to verify its effectiveness. 
% \linli{we can list the names of all baselines here, and if space permits, categorize the methods.}
% These baselines are primarily categorized into three types: classic VLM-based policies, generative policies based on diffusion or flow matching, and recent VLA variants focusing on enhanced spatio-temporal reasoning. Please refer to Appendix~\ref{sec:appendix_baselines} for detailed descriptions of all compared methods.

\subsection{Simulation Evaluation}

\noindent \textbf{Evaluation on LIBERO.} As shown in Table~\ref{tab:libero}, $\mathtt{FibVLA}$ demonstrates competitive performance across all benchmarks. Notably, it achieves a significant improvement on the LIBERO-Long suite, outperforming the second-best baseline by {7.21\%} and reaching a success rate of 95.2\%.
This improvement is primarily attributed to $\mathtt{FibVLA}$’s gradual-frequency sampling, which effectively bridges the gap between capturing global task goals and local motion dynamics. In contrast, conventional policies often struggle with these multi-stage tasks, as they tend to lose short-term precision while maintaining long-term context.
% While due to the challenges of maintaining long-term context without losing short-term precision, FibVLA achieves a 95.2\% success rate. 
(Detailed success rates for individual subtasks are provided in Appendix~\ref{sec:libero-long}). 
Beyond long-horizon scenarios, $\mathtt{FibVLA}$ also maintains high stability in basic manipulation tasks involving spatial, object, and goal-oriented operations, 
% yielding success rates of 97.8\%, 98.0\%, and 96.4\%, respectively. 
These results indicate that the temporal encoding module not only preserves the model’s spatial reasoning capabilities but may also enhance accuracy by improving temporal consistency in action outputs. 
Overall, $\mathtt{FibVLA}$ achieves an average success rate of 96.8\%, demonstrating its effectiveness as a general-purpose manipulation policy in handling both complex reasoning and basic execution tasks.
\input{table/LIBERO}

\input{table/Simpler-fractal}
\noindent \textbf{Evaluation on SimplerEnv-Fractal.} To validate the model's robustness against visual perturbations, we conduct a comprehensive evaluation on SimplerEnv-Fractal using two core protocols: Visual Matching, which aims to minimize the visual appearance gap between simulation and real world to assess transfer potential; and Visual Aggregation, serving as a stress test by introducing extensive environmental variations to evaluate adaptability to visual domain shifts. 
Table~\ref{tab:fractal} reports the overall performance: $\mathtt{FibVLA}$ achieves the 5.87\% improvement in average success rate compared to the second-best baseline, CogACT. 
Notably, under the rigorous Visual Aggregation setting, despite severe environmental changes, $\mathtt{FibVLA}$ maintains a stable average success rate of 65.5\% across all four subtasks without any significant weaknesses. 
This finding indicates that $\mathtt{FibVLA}$ goes beyond mere memorization of training data patterns. Instead, it leverages temporal context to comprehend task semantics, thereby maintaining stable control amid lighting variations and background distractions, effectively mitigating the action jitter or execution failures often observed in baseline models lacking temporal modeling.

\noindent \textbf{Evaluation on SimplerEnv-Bridge.} In the SimplerEnv-Bridge setting, we evaluate the model's manipulation capabilities on the WidowX robot, covering tasks ranging from basic pick-and-place to complex long-horizon manipulation. 
As shown in Table~\ref{tab:bridge}, $\mathtt{FibVLA}$ demonstrates superior performance on the Bridge benchmark, achieving an average success rate of 67.3\%. Compared to CogACT and the recent SOTA model $\pi_0$, $\mathtt{FibVLA}$ achieves significant performance gains of 3.12\% and 20.9\%, respectively. These results indicate that $\mathtt{FibVLA}$'s temporal module effectively models dynamic object interactions, substantially enhancing the overall execution accuracy of policies in Bridge tasks.
\input{table/Simpler-bridge}

\noindent \textbf{Evaluation on MIKASA-Robo.} To investigate the limits of $\mathtt{FibVLA}$ in handling memory-intensive tasks, we incorporate the MIKASA-Robo benchmark for validation. Distinct from the previously discussed long-horizon tasks that emphasize multi-step logical reasoning, MIKASA-Robo focuses on evaluating the model's capacity to retain object and spatial memory. 
In this benchmark, the environment only briefly presents the target configuration in the initial few frames as a cue, after which the scene resets to a default state. Consequently, the model must rely solely on the memory of the initial keyframes to reconstruct the target scene, as the target information is unavailable in the current observations.
As shown in Table~\ref{tab:mikasa}, $\mathtt{FibVLA}$ achieves superior performance with an average success rate of 46.5\%, outperforming $\pi_0$ by 40.9\%. 
Notably, the model demonstrates a significant performance advantage in specific subtasks that involve transient visual goals. This improvement is likely attributed to the design of the temporal module. In scenarios where target positions or task-relevant cues are only briefly visible at the onset, the gradual-frequency sampling mechanism prioritizes the retention of these critical initial states. By maintaining these historical visual features, $\mathtt{FibVLA}$ is better equipped to track objectives during the subsequent manipulation phase, effectively mitigating the goal-forgetting issues frequently observed in traditional VLA models.
\input{table/Mikasa}

\subsection{Real-world Evaluation}
We have conducted extensive evaluations on the Piper robotic platform, equipped with both third-person and wrist-mounted cameras, to validate the effectiveness and robustness of $\mathtt{FibVLA}$ in real-world settings. The model is fine-tuned on teleoperated demonstrations covering single/dual-arm coordination across various horizons. To move beyond binary success rates, we introduce a fine-grained step-wise scoring metric to evaluate manipulation quality and execution stability. 
The detailed definitions and task decompositions are provided in Appendix \ref{sec:rubric}.

\begin{figure}[htbp] %
    \centering
    \includegraphics[width=0.48\textwidth]{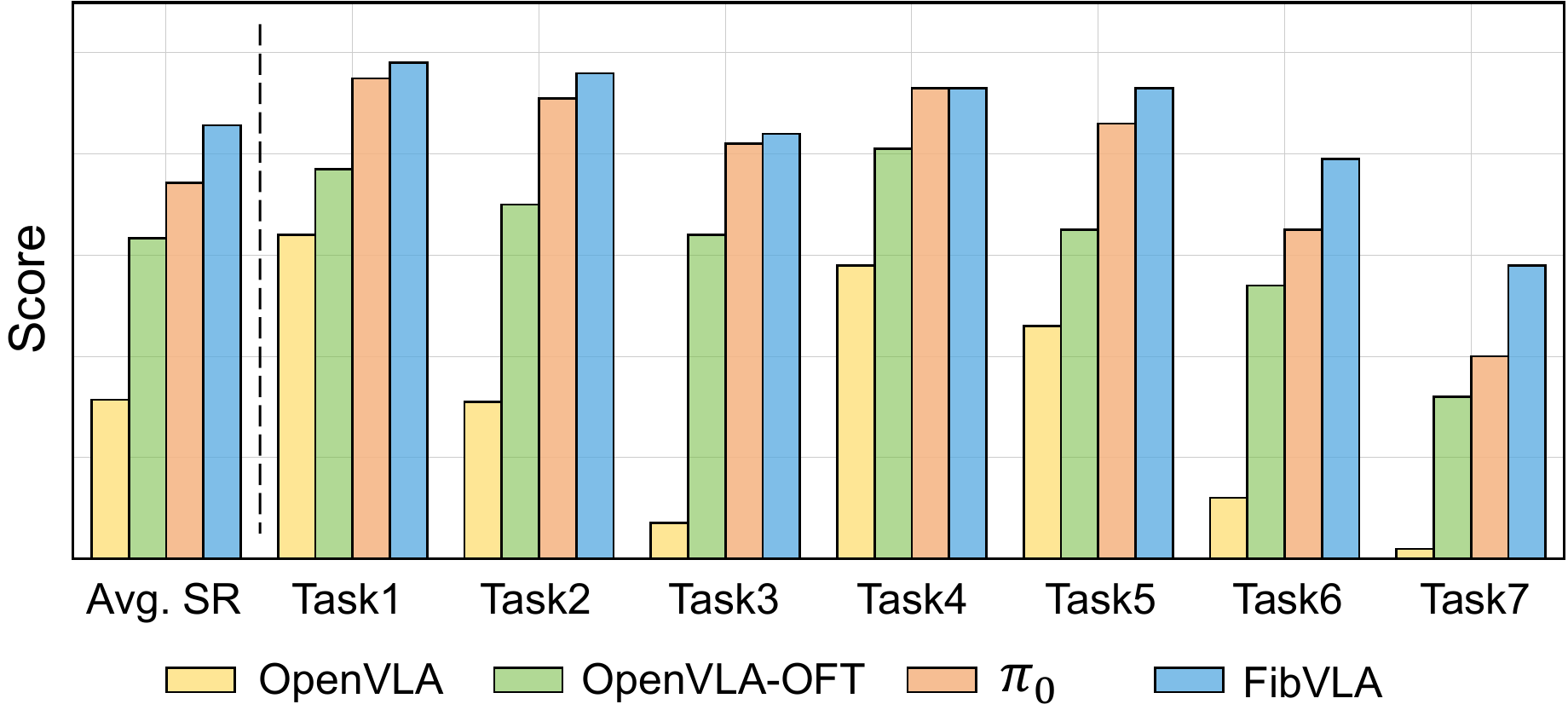} % 设置图片宽度
    \caption{Real-world performance comparison across individual tasks. Detailed subtask names are provided in Appendix~\ref{sec:subtasks}.}
    \label{fig:real-world}
    \vspace{-0.1cm}
\end{figure}

We compare $\mathtt{FibVLA}$ against mainstream baselines, including OpenVLA, OpenVLA-OFT, and $\pi_0$. As shown in Figure~\ref{fig:real-world}, 
$\mathtt{FibVLA}$ demonstrates superior performance across all tasks, achieving an average score of 85.7, which significantly outperforms the second-best performing model $\pi_0$ (by a margin of +11.4 points). 
Specifically, $\mathtt{FibVLA}$ not only exhibits exceptional stability and fluency with near-perfect scores on the first five basic tasks, but also establishes a substantial lead on the more challenging long-horizon tasks (Task 6 and Task 7). While other baseline models suffer significant performance degradation due to cumulative errors, $\mathtt{FibVLA}$ maintains robust performance by leveraging effective memory of historical states. 
These results validate that the temporal consistency provided by our module is not merely a theoretical advantage in simulation, serving as a critical factor in overcoming perceptual noise and dynamic uncertainties in the physical world, thereby enabling robust real-world deployment of complex long-horizon tasks. {Further details regarding the physical experimental setup and specific evaluation cases are described in Appendix~\ref{appendix: experiment}.}

\subsection{Ablation Study}
To validate the effectiveness of $\mathtt{FibVLA}$'s core components, we conduct two variants: 1) \textbf{w/o Sampling}, which removes the progressive frequency sampling strategy from $\mathtt{FibVLA}$. 2) \textbf{w/o CTE}, which removes the Channel-wise Temporal Encoding module. Ablation experiments have conducted on the LIBERO-Long benchmark and the real-world dataset. The experimental results are shown in Table~\ref{tab:module-ablation}.
Specifically, $\mathtt{FibVLA}$ significantly outperforms all variants. The w/o Sampling variant suffers the most severe performance degradation, with the success rate on LIBERO-Long dropping to 88.4\% and the real-world score decreasing to 78.3\%. 
This demonstrates that an efficient historical information compression mechanism is crucial for capturing long-horizon dependencies. 
Furthermore, w/o CTE also causes a noticeable performance loss (a decline of approximately 4\% in success rate), validating the key role in temporal feature extraction and fusion.
\input{table/Ablation}

\subsection{Efficiency Studiy}
To evaluate the inference efficiency of different sampling strategies, we compare the inference latency of our Fibonacci recurrent inference strategy against other common mathematical sampling methods, including normal Logarithmic sampling (\textbf{w/ Logarithmic}) and long-short term hybrid sampling (\textbf{w/ Long–Short}).  We also included recent state-of-the-art temporal VLA methods, TraceVLA and HiF-VLA, both of which employ traditional continuous time-window sampling mechanisms. 
For fair comparison, all methods have configured with a uniform historical time window size of 10 frames. 
% As shown in Table~\ref{tab:inference_time}, the results demonstrate the superior computational efficiency of Fibonacci sampling, which requires only 177 ms per inference, which corresponds to a {9.69\% and 27.16\%} reduction in inference time compared to TraceVLA and HiF-VLA, respectively. 
As shown in Table~\ref{tab:inference_time}, the results on the LIBERO-Long task suite demonstrate the superior computational efficiency of Fibonacci sampling. While maintaining high-level performance across these long-horizon tasks, our method requires only 177 ms per inference, achieving a reduction in latency of 9.69\% and 27.16\% compared to TraceVLA and HiF-VLA, respectively.
$\mathtt{FibVLA}$ also outperforms other sampling variants, indicating that the Fibonacci recurrent inference maximizes historical information coverage with minimal temporal overhead while reducing computational latency, thereby offering significant advantages for real-time robotic deployment.

\input{table/Inference}

%% file: table/LIBERO.tex
\begin{table}[t]

\centering

\renewcommand{\arraystretch}{1.1}

\setlength{\tabcolsep}{3pt}

\caption{Comparison of success rates on the LIBERO. The \textbf{best} results are highlighted in bold, and the \underline{second-best} results are underlined. This convention is applied to all subsequent tables.}

\begin{tabular}{lccccc}

\toprule

\multicolumn{1}{l}{\textbf{Method}} &

\makecell{\textbf{Avg. SR}} &

\textbf{Spatial} &

\textbf{Object} &

\textbf{Goal} &

\textbf{Long} \\

\midrule

% 数据已重新排列：将原本最后一列的平均分移到第二列

Octo     & 75.1 & 78.9 & 85.7 & 84.6  & 51.1  \\
$\pi_0$       & \underline{94.2} & 96.8 & \textbf{98.8} & \underline{95.8} & 85.2  \\
OpenVLA       & 76.5 & 84.7  & 88.4  & 79.2  & 53.7 \\
CogACT   & 93.5 & \underline{97.2} & \underline{98.0} & 90.2 & \underline{88.8}  \\
TraceVLA      & 74.8 & 84.6 & 85.2 & 75.1 &  54.1 \\
SpatialVLA    & 78.1 & 88.2 & 89.9 & 78.6 & 55.5 \\
4D-VLA        & 88.6 & 88.9 & 95.2 & 90.9 & 79.1 \\
\rowcolor{gray!15}
$\mathtt{FibVLA}$ & \textbf{96.8} & \textbf{97.8} & \underline{98.0} & \textbf{96.4} & \textbf{95.2}  \\
\bottomrule
\end{tabular}

\label{tab:libero}
\vspace{-0.5cm}
\end{table}

%% file: table/Simpler-fractal.tex
\begin{table*}[t]
\centering
\renewcommand{\arraystretch}{1.1}
\setlength{\tabcolsep}{4pt}
\caption{Comparison of success rates on the SimplerEnv-Fractal, reporting success rates across Visual Matching (VM) and Visual Aggregation (VA) settings. Detailed subtask names for this and subsequent evaluations are provided in Appendix~\ref{sec:subtasks}.}
\label{tab:fractal} % <--- 建议移到这里，紧跟 caption

\begin{tabular}{l|ccccc|ccccc|c} 
\toprule

\multirow{2}{*}{\textbf{Method}} & 
\multicolumn{5}{c|}{\textbf{Visual Matching (VM)}} & 
\multicolumn{5}{c|}{\textbf{Visual Aggregation (VA)}} & 
\multirow{2}{*}{\textbf{Overall}} \\ 

 & \textbf{Avg. SR} & \textbf{Task1} & \textbf{Task2} & \textbf{Task3} & \textbf{Task4} 
 & \textbf{Avg. SR} & \textbf{Task1} & \textbf{Task2} & \textbf{Task3} & \textbf{Task4} & \\ 
\midrule
RT-1-X     & 42.4 & 56.7 & 31.7 & 59.7 & 21.3 & 30.2 & 49.0 & 32.3 & 29.4 & 10.1 & 36.3 \\
RT-2-X     & 46.3 & 78.7 & 77.9 & 25.0 & 3.7  & 54.4 & \underline{82.3} & 79.2 & 35.5 & 20.6 & 50.4 \\
OpenVLA    & 34.3 & 18.0 & 56.3 & 63.0 & 0.0  & 39.3 & 60.8 & 67.7 & 28.8 & 0.0  & 36.8 \\
Octo  & 11.0 & 17.0 & 4.2  & 22.7 & 0.0  & 1.2  & 0.6  & 3.1  & 1.1  & 0.0  & 6.1 \\
$\pi_0$    & 69.1 & 88.0 & 80.3 & 56.0 & \textbf{52.2} & --   & --   & --   & --   & --   & --   \\
CogACT     & \underline{74.8} & \textbf{91.3} & 85.0 & \underline{71.8} & \underline{50.9} & \underline{61.3} & \textbf{89.6} & \underline{80.8} & 28.3 &\underline{46.6} & \underline{68.1} \\
TraceVLA   & 45.8 & 45.0 & 63.8 & 63.1 & 11.1 & 49.8 & 64.3 & 60.6 & \textbf{61.6} & 12.5 & 47.8 \\
SpatialVLA & 56.0 & 79.3 & \textbf{90.0} & 54.6 & 0.0  & 51.8 & 78.7 & \textbf{83.0} & 39.2 & 6.3  & 53.9 \\
\rowcolor{gray!15}
$\mathtt{FibVLA}$ & \textbf{78.6} & \underline{89.0} & \underline{86.7} & \textbf{86.5} & \textbf{52.2} & \textbf{65.5} & 82.1 & 77.9 & \underline{43.6} & \textbf{58.3} & \textbf{72.1} \\
\bottomrule
\end{tabular}%
\vspace{-0.2cm}
\end{table*}

%% file: table/Simpler-bridge.tex
\begin{table}[t]
\centering
\renewcommand{\arraystretch}{1.1}
\setlength{\tabcolsep}{2pt}
\caption{Comparison of success rates on the SimplerEnv-Bridge.}
\begin{tabular}{lccccc}
\toprule
% 表头：Method | Avg | Task 1 | Task 2 | Task 3 | Task 4
\multicolumn{1}{l}{\textbf{Method}} &
\makecell{\textbf{Avg. SR}} &
\textbf{Task 1} &
\textbf{Task 2} &
\textbf{Task 3} &
\textbf{Task 4} \\
\midrule
% 数据已重新排列：将原本最后一列的平均分移到第二列

RT-1-X        & 1.1  & 0.0  & 4.2  & 0.0  & 0.0   \\
OpenVLA       & 4.2  & 4.2  & 0.0  & 0.0  & 12.5  \\
Octo     & 17.5 & 15.8 & 12.5 & 0.0  & 41.7  \\
$\pi_0$       & \underline{55.7} & 63.3 & \underline{58.8} & 21.3 & 79.2  \\
CogACT   & 51.3 & \underline{71.7} & 50.8 & 15.0 & 67.5  \\
TraceVLA      & 27.7 & 12.5 & 16.6 & 16.6 & 65.0  \\
SpatialVLA    & 42.7 & 16.7 & 25.0 & \underline{29.2} & \textbf{100} \\
\rowcolor{gray!15}
$\mathtt{FibVLA}$ & \textbf{67.3} & \textbf{75.0} & \textbf{65.4} & \textbf{35.2} & \underline{93.7}  \\
\bottomrule
\end{tabular}
\vspace{-0.5cm}
\label{tab:bridge}
\end{table}

%% file: table/Mikasa.tex
\begin{table}[t]
\centering
\renewcommand{\arraystretch}{1.1}
\setlength{\tabcolsep}{3pt}
\caption{Comparison of success rates on the MIKASA-Robo.}
% 调整为 lccccc (共6列)
\begin{tabular}{lccccc}
\toprule
% 表头删除了 Task3
\textbf{Method} & \textbf{Avg. SR} & \textbf{Task1} & \textbf{Task2} & \textbf{Task3} & \textbf{Task4} \\
\midrule
% 数据删除了原 Task3 的那一列
$\pi_0$      & \underline{33.0} & 33.0 & \textbf{42.0} & \underline{31.0} & \underline{26.0} \\
SpatialVLA   & 22.0 & 23.0 & 27.0 & 18.0 & 20.0 \\
OpenVLA-OFT  & 26.5 & \underline{48.0} & 14.0 & 27.0 & 17.0 \\
\rowcolor{gray!15}
$\mathtt{FibVLA}$ & \textbf{46.5} & \textbf{78.0} & \underline{37.0} & \textbf{36.0} & \textbf{35.0} \\
\bottomrule
\end{tabular}
\vspace{-0.4cm}
\label{tab:mikasa}
\end{table}

%% file: table/Ablation.tex
\begin{table}[t]
\centering
\renewcommand{\arraystretch}{1.1} % 稍微增加行高，防止文字挤在一起
\setlength{\tabcolsep}{6pt}      % 增加列间距，看起来更宽敞

\caption{Ablation study of $\mathtt{FibVLA}$ components on LIBERO-Long and Real-world dataset.}
\label{tab:module-ablation}

\begin{tabular}{lcc}
\toprule
% 表头部分
% \makecell[l] 表示左对齐，默认垂直居中
% \makecell[c] 表示居中对齐，\\ 用于换行
\makecell[l]{\textbf{Variant}} & 
\makecell{\textbf{LIBERO-Long} \\ \textbf{(Avg. SR)}} & 
\makecell{\textbf{Real-world} \\ \textbf{(Avg. Score)}} \\
\midrule
% 第一行：FibVLA (Ours)
\rowcolor{gray!15}
$\mathtt{FibVLA}$    & \textbf{95.2} & \textbf{85.7} \\
\hspace{2mm} --(w/o Sampling) & 88.4 & 78.3 \\
\hspace{2mm} --(w/o CTE)      & 91.2 & 80.0 \\

\bottomrule
\end{tabular}
\vspace{-0.5cm}
\end{table}

%% file: table/Inference.tex
\begin{table}[t]
\centering
% 保持行高设置
\renewcommand{\arraystretch}{1.1} 
% [列间距调整] 因为只有3列，建议设大一点（如15pt-20pt），让表格看起来更舒展
\setlength{\tabcolsep}{3pt}      
\caption{Comparison of inference time across different sampling strategies, including both competitive baselines and $\mathtt{FibVLA}$ variants utilizing alternative sampling curves.}
% lcc 表示：第1列左对齐，后2列居中
\begin{tabular}{lccc}
\toprule
\textbf{Sampling Method} & \makecell{\textbf{Inference Time} \\ \textbf{(ms)}} & \makecell{\textbf{Avg. SR} \\ \textbf{(LIBERO-Long)}} \\
\midrule
  \rowcolor{gray!15} $\mathtt{FibVLA}$  & \textbf{177} & \textbf{95.2}\\
  \hspace{2mm} --(w/ Logarithmic) &  201 &  94.6\\
  \hspace{2mm} --(w/ Long–Short)  &  235 & \underline{95.5}\\
  TraceVLA &  \underline{196} & 54.1\\
  HiF-VLA &   243 & 94.4\\
\bottomrule
\end{tabular}
\vspace{-0.5cm}
\label{tab:inference_time}
\end{table}

%% file: table/Benchmark.tex
\begin{itemize}[leftmargin=*]
    \label{sec:appendix_benchmarks}
    \item \textbf{LIBERO~\cite{liu2023libero}:} A lifelong robotic learning benchmark designed to decouple knowledge transfer. Here, LIBERO-Spatial, Object, and Goal (10 tasks each) investigate the transfer of spatial relationships, object types, and motion behaviors, respectively, while LIBERO-Long comprises 10 challenging long-horizon tasks for evaluating downstream performance in complex, multi-step scenarios.
    
    \item \textbf{MIKASA-Robo~\cite{cherepanov2025mikasa}:} A comprehensive benchmark comprising 32 memory-intensive manipulation tasks across 12 categories, designed to evaluate four core memory capabilities: Object Memory, Spatial Memory, Sequential Memory, and Memory Capacity. By targeting these distinct cognitive aspects, this benchmark provides a rigorous environment for assessing a model’s generalization in handling complex geometries and contact-rich interactions.
    
    \item \textbf{SimplerEnv~\cite{li2025simpler}:} An open-source simulation suite for evaluating real-to-sim transfer, encompassing robotic setups from Google Robot (RT-series) and BridgeData V2. It offers a standardized Gym interface for seamless task interaction and provides integrated inference pipelines for evaluating generalist policies like RT-1 and Octo.
    \begin{itemize}
        \item \textbf{Bridge V2~\cite{walke2023bridgedata}:} A large-scale real-world dataset comprising over 60k trajectories across diverse environments using the WidowX 250 manipulator. As a key component of the Open X-Embodiment cross-embodiment dataset, Bridge V2 serves as a primary training source for verifying a model's effectiveness in open-vocabulary instruction following, with performance further validated in the aligned SimplerEnv.
    
        \item \textbf{Fractal~\cite{brohan2023rt}:} A large-scale real-world robotics dataset released alongside the RT-1 model, consisting of approximately 130k episodes and over 700 tasks collected by a fleet of 13 robots over 17 months. Characterized by its high diversity in objects, backgrounds, and manipulation skills, Fractal serves as a foundational data source for pre-training generalist VLA models and is a core component of the Open X-Embodiment dataset.
     \end{itemize}    

    \item \textbf{Real-world:} {The real-world benchmarking is conducted on a curated custom dataset comprising over 600k frames across 15 tasks, encompassing both single- and dual-arm manipulations with varying temporal horizons. To prioritize experimental reproducibility, a standardized environment is established featuring a consistent white background and black tablecloth, with interaction objects limited to easily accessible props such as various cubes and common household items. A subset of 7 representative tasks was selected from this suite for comprehensive performance evaluation.}

\end{itemize}

%% file: table/Baseline.tex
\begin{itemize}[leftmargin=*]
\label{sec:appendix_baselines}
    \item \textbf{RT-1-X~\cite{brohan2023rt} \& RT-2-X~\cite{zitkovich2023rt}:} \textbf{RT-1-X} is a transformer-based robot action model that uses EfficientNet as a vision encoder to output discretized action tokens. \textbf{RT-2-X} is a representative VLA model that transforms a Vision-Language Model (VLM) into a closed-loop robotic policy through co-fine-tuning on large-scale internet data and robotic trajectories.

    \item \textbf{OpenVLA~\cite{kim2025openvla} \& OpenVLA-OFT~\cite{kim2025fine}:} {\textbf{OpenVLA} is a Llama 2-based model using Prismatic encoders to generate actions via discrete token quantization. Its successor, \textbf{OpenVLA-OFT}, optimizes this framework by adopting a continuous action head with $L_1$ regression and action chunking. This transition from autoregressive tokenization to direct regression enables higher control frequencies and improved inference efficiency while maintaining the original model's strong generalization.}
    
    \item \textbf{Octo~\cite{ghosh2024octo}:} A transformer-based generalist robot policy pre-trained on the Open X-Embodiment dataset. Unlike the RT series, Octo employs a diffusion head to output multi-modal action distributions rather than discrete tokens.

    \item \textbf{$\pi_0$~\cite{black2024pi_0}:} A foundation model introduced by Physical Intelligence. Similar to $\mathtt{FibVLA}$, it leverages flow matching techniques to model continuous action distributions. As a powerful generalist policy, it represents the frontier of generative action modeling.
    
    \item \textbf{CogACT~\cite{li2024cogact}:} A VLA model designed to distinguish ``cognition'' from ``action.'' By decoupling Chain-of-Thought (CoT) reasoning from specific action execution, it enhances task planning and execution capabilities for complex long-horizon tasks.
    
    \item \textbf{TraceVLA~\cite{zhengtracevla}:} This method explicitly generates ``visual action traces'' via visual prompts to guide the VLA model in understanding the end-effector's motion path, thereby improving manipulation precision.
    
    \item \textbf{SpatialVLA~\cite{qu2025spatialvla}:} Focuses on addressing the lack of spatial awareness in VLA models. It injects geometric knowledge into the VLM backbone by co-training action prediction with 3D spatial reasoning tasks, such as depth prediction or 3D bounding box regression.
    
    \item \textbf{4D-VLA~\cite{zhang20254d}:} Further introduces the temporal dimension on top of 3D perception. By processing dynamic 3D point cloud video streams, it enables the model to understand dynamic changes of objects in space-time, generating spatiotemporally consistent action trajectories.

    \item \textbf{HiF-VLA~\cite{lin2025hif}:} A unified framework that leverages motion representations for bidirectional temporal reasoning, encompassing hindsight priors and foresight anticipation. It addresses temporal myopia in VLA models by capturing inter-state dynamics to filter static background noise effectively.
\end{itemize}

%% file: table/Instruction.tex
\begin{table}[ht]
\centering
\renewcommand{\arraystretch}{1.1} % [行间距] 1.0是默认，1.1稍微宽松一点
\setlength{\tabcolsep}{3pt}       % [列间距] 默认是6pt。因为你有12列，建议设为 3pt
\caption{Mapping between Task IDs and language instructions across evaluated benchmarks.}
\label{tab:benchmark_results}
% 调整列宽以适应页面，可根据需要修改 'c' 的个数或使用 X 列
\begin{tabular}{lccccc}
\toprule
\textbf{Task ID} & \textbf{LIBERO-Long} & \textbf{\makecell{SimplerEnv\\-Bridge}} & \textbf{\makecell{SimplerEnv\\-Fractal}} & \textbf{MIKASA-Robo} & \textbf{\makecell{Real-world}} \\
\midrule
Task 1  & \makecell{Put soup and \\ box in basket} & \makecell{PutSpoonOn\\TableCloth} & \makecell{GraspSingle\\OpenedCokeCan} & ShellGameTouch & Place Bowl (Test) \\
Task 2  & \makecell{Put box and \\ butter in basket} & \makecell{PutCarrotOnPlate} & \makecell{MoveNearGoogle\\BakedTex} & InterceptMedium & Place Bowl (Color) \\
Task 3  & \makecell{Turn on stove \\ and put pot} & \makecell{StackGreenCubeOn\\YellowCubeBakedTex} & \makecell{Open/CloseDrawer} & TakeItBack & Push Cube (Shape) \\
Task 4  & \makecell{Put bowl in \\ drawer and close} & \makecell{PutEggplant\\InBasket} & \makecell{OpenTopDrawer\\andPlaceApple} & RememberShape5&  Pick Cube\\
Task 5  & \makecell{Put mugs on left \\ and right plates} & - & - & - & Dual-Arm Handover \\
Task 6  & \makecell{Pick book and \\ place it in back} & - & - & - & \makecell{Stack Bowls \\(Instruction)} \\
Task 7  & \makecell{Put mug on plate, \\ pudding right} & - & - & - & Stack Bowls (Visual) \\
Task 8  & \makecell{Put soup and \\ sauce in basket} & - & - & - & - \\
Task 9  & \makecell{Put both pots \\ on stove} & - & - & - & - \\
Task 10 & \makecell{Put mug in micro-\\wave and close} & - & - & - & - \\
\bottomrule
\end{tabular}
\end{table}

%% file: table/LIBERO-10.tex
\begin{table*}[htbp]
\centering
% --- 调整参数区域 ---
\renewcommand{\arraystretch}{1.1} % [行间距] 1.0是默认，1.1稍微宽松一点
\setlength{\tabcolsep}{4pt}       % [列间距] 默认是6pt。因为你有12列，建议设为 3pt 到 5pt 之间，防止超宽
% ------------------
\caption{Comparison of success rates on LIBERO-Long subtasks.}
\begin{tabular}{lccccccccccc}
\toprule
\textbf{Method} &
  \textbf{\makecell{Avg. SR}} & 
  \textbf{Task1} &
  \textbf{Task2} &
  \textbf{Task3} &
  \textbf{Task4} &
  \textbf{Task5} &
  \textbf{Task6} &
  \textbf{Task7} &
  \textbf{Task8} &
  \textbf{Task9} &
  \textbf{Task10} \\ \midrule
OpenVLA        & 54.0 & 35.0 & 95.0 & 65.0 & 45.0 & 40.0 & 80.0 & 60.0 & 45.0 & 20.0 & 55.0 \\
OpenVLA-OFT    & 91.0 & 82.0 & \underline{96.0} & \underline{96.0} & 94.0 & 90.0 & 96.0 & \underline{92.0} & \textbf{100}  & 70.0 & \textbf{94.0} \\
$\pi_0$        & 72.0 & 86.0 & 86.0 & 60.0 & 86.0 & 82.0 & 52.0 & 84.0 & 82.0 & 26.0 & 76.0 \\
UniVLA     & 63.0 & 64.0 & 82.0 & 76.0 & 96.0 & 58.0 & \underline{98.0} & 24.0 & 74.0 & 32.0 & 26.0 \\
HiF-VLA        & \underline{94.4} & \underline{94.0} & \textbf{98.0} & \textbf{100}  & \textbf{100}  & \underline{94.0} & \textbf{100}  & 90.0 & \underline{98.0} & \underline{76.0} & \textbf{94.0} \\ 
\rowcolor{gray!15} 
$\mathtt{FibVLA}$   & \textbf{95.2} & \textbf{96.0} & \textbf{98.0} & \underline{96.0} & \underline{98.0} & \textbf{96.0} & 96.0 & \textbf{98.0} & \textbf{100}  & \textbf{88.0} & \underline{86.0} \\ \bottomrule
\end{tabular}%

\label{tab:main_results}
\end{table*}

%% file: table/Rubric.tex
% === 样式设置 ===
\small
\renewcommand{\arraystretch}{1.3} 
\renewcommand{\tabularxcolumn}[1]{m{#1}} % 保持内容垂直居中
% === 开始跨页表格 ===
% 注意：这里不再使用 \begin{table*}，直接用 xltabular
\begin{xltabular}{\textwidth}{@{} >{\raggedright\arraybackslash}m{4.5cm} >{\raggedright\arraybackslash}X >{\centering\arraybackslash}m{1.6cm} >{\centering\arraybackslash}m{1.6cm} @{}}

% === 1. 标题和标签 (放在最前面) ===
\caption{List of Self-collected Real-world Evaluation Tasks with Decomposed Action Steps} \label{tab:rubric} \\

% === 2. 第一页的表头 ===
\toprule 
\textbf{Task \& Instruction} & 
\textbf{Evaluation Steps (5 Sub-steps)} & 
\textbf{Config} & 
\textbf{Horizon} \\
\midrule 
\endfirsthead

% === 3. 续页的表头 (第二页及以后显示的表头) ===
\multicolumn{4}{c}{\small \textit{... Continued from previous page}} \\
\toprule
\textbf{Task \& Instruction} & 
\textbf{Evaluation Steps (5 Sub-steps)} & 
\textbf{Config} & 
\textbf{Horizon} \\
\midrule
\endhead

% === 4. 表格底部的结束线 ===
\bottomrule
\endlastfoot

% === 5. 表格内容 ===

\textbf{Place Bowl (Test):} Put the blue plate on desk, pick up the white bowl, put the bowl on the plate. & 
1. Move to hover directly over the white bowl. \newline
2. Descend vertically and close the gripper to clamp the bowl. \newline
3. Lift the bowl and move horizontally to the plate. \newline
4. Descend until the bowl bottom touches the plate. \newline
5. Open the gripper and retract vertically. & 
Single & Short \\
\midrule

\textbf{Place Bowl (Spatial):} Put the blue plate on desk, pick up the white bowl, put the bowl on the plate. & 
1. Move to the bowl's random position and align the wrist. \newline
2. Descend and close the gripper upon contact. \newline
3. Lift and reorient the wrist in mid-air to level the bowl. \newline
4. Move carrying the bowl to the coordinates of the plate. \newline
5. Lower the bowl onto the plate and open the gripper. & 
Single & Short \\
\midrule

\textbf{Place Bowl (Color):} Put the bowl on the color [X] plate. & 
1. Move to hover over the white bowl. \newline
2. Descend, close the gripper, and lift the bowl. \newline
3. Move the arm to the position of the [Color] plate. \newline
4. Descend vertically to place the bowl into the plate. \newline
5. Open the gripper and retract the arm. & 
Single & Short \\
\midrule

\textbf{Push Cube (Shape):} Push the shape [X] block into the red region. & 
1. Move directly over the specific [Shape] block. \newline
2. Descend and align with the contact surface behind the object. \newline
3. Rotate the gripper to the pushing angle. \newline
4. Push the block horizontally into the red region. \newline
5. Lift the arm vertically to finish the task. & 
Single & Short \\
\midrule

\textbf{Pick Cube:} Pick up the left/middle/right one among the three cubes and put it into the bowl. & 
1. Move directly above the cube matching the [Position]. \newline
2. Descend and close the gripper to grab the cube. \newline
3. Lift the cube and move above the bowl. \newline
4. Open the gripper to let the cube fall into the bowl. \newline
5. Return the arm to the initial home position. & 
Single & Short \\
\midrule

\textbf{Pick Functional Object:} Put the object used for [X] on the plate. & 
1. Move to the center of mass of the target functional object. \newline
2. Rotate the wrist to a suitable grasping angle and close gripper. \newline
3. Lift the object vertically off the table. \newline
4. Move horizontally to the plate's position. \newline
5. Lower until contact is made and open the gripper. & 
Single & Short \\
\midrule

\textbf{Wipe Whiteboard:} Wipe the whiteboard clean. & 
1. Move to the eraser, descend, and close the gripper tightly. \newline
2. Move the eraser to the starting corner of the writing. \newline
3. Press down to ensure contact between eraser and board. \newline
4. Execute a wiping motion (zigzag or linear) across the marks. \newline
5. Lift the eraser vertically off the board. & 
Single & Short \\
\midrule

\textbf{Stack Bowls (Instruction):} Stack bowls in the order of color [X], [Y], [Z] from top to bottom. & 
1. Grasp bowl [Z], place it in the workspace center, release. \newline
2. Grasp bowl [Y], move it above [Z], align, and release. \newline
3. Grasp bowl [X], move it above [Y]. \newline
4. Align carefully and lower [X] onto [Y]. \newline
5. Open gripper and retreat to complete the tower. & 
Single & Long \\
\midrule

\textbf{Stack Bowls (Visual):} Stack bowls in the same/reverse order as the stack on the left. & 
1. Grasp the bottom-matching bowl and place on the target spot. \newline
2. Grasp the middle-matching bowl from the supply. \newline
3. Stack it precisely onto the bottom bowl and release. \newline
4. Grasp the top-matching bowl from the supply. \newline
5. Stack it onto the middle bowl and release. & 
Single & Long \\
\midrule

\textbf{Dual-Arm Handover:} Pick up cylinder with left arm, pass to right arm. & 
1. Left arm moves to the cylinder and closes the gripper. \newline
2. Left arm lifts object to center; Right arm moves to meet it. \newline
3. Right arm closes gripper on the free end of the cylinder. \newline
4. Left arm opens gripper to release the object. \newline
5. Right arm moves away with the object; Left arm retracts. & 
Dual & Short \\
\midrule

\textbf{Classification:} Put all the fruits into the bowl. & 
1. Move to the first fruit, grasp it, and lift. \newline
2. Move to the bowl and open gripper to release. \newline
3. Move to the second fruit, grasp it, and lift. \newline
4. Move to the bowl and open gripper to release. \newline
5. Repeat until no fruits remain, then return to home. & 
Single & Long \\
\midrule

\textbf{Mid-air Stacking:} Place the bowl on the plate in mid-air. & 
1. Left arm grasps the plate; Right arm grasps the bowl. \newline
2. Arms lift and meet in the center; Left holds plate flat. \newline
3. Right arm positions the bowl directly above the plate. \newline
4. Right arm lowers gently until the bowl touches the plate. \newline
5. Right arm opens gripper and moves away. & 
Dual & Short \\
\midrule

\textbf{Uncap Marker:} Remove the cap from the marker. & 
1. Left arm moves to hold the marker body firmly on the table. \newline
2. Right arm moves to the marker cap. \newline
3. Right arm closes gripper to clamp the cap. \newline
4. Right arm pulls horizontally along the axis to detach the cap. \newline
5. Right arm places the cap on the table. & 
Dual & Short \\
\midrule

\textbf{Make Coffee:} Make a cup of instant coffee. & 
1. Grip the powder cup, pour into the empty cup, and return it. \newline
2. Move to the spoon and grip the handle. \newline
3. Insert the spoon vertically into the filled cup. \newline
4. Perform a circular stirring motion with the wrist. \newline
5. Lift the spoon out of the cup. & 
Dual & Long \\
\midrule

\textbf{Fold Clothes:} Fold the clothes on the table. & 
1. Move to the left sleeve, grip, fold inward, and release. \newline
2. Move to the right sleeve, grip, fold inward, and release. \newline
3. Move to the bottom hem of the shirt. \newline
4. Grip and lift the hem towards the collar. \newline
5. Release to complete the fold and retract arms. & 
Dual & Long \\

\end{xltabular}

%% file: table/experimental_para.tex
\begin{table}[htbp]
\centering
\caption{Hyperparameter settings and data statistics for different evaluation environments.}
\label{tab:hyperparameters}
\resizebox{\textwidth}{!}{% 自动缩放表格以适应文本宽度
\begin{tabular}{llccccc}
\toprule
\textbf{Category} & \textbf{Hyperparameter} & \textbf{LIBERO} & \textbf{MIKASA-Robo} & \textbf{SimplerEnv-Bridge} & \textbf{SimplerEnv-Fractal} & \textbf{Real-world} \\ 
\midrule
\multirow{9}{*}{\textbf{Optimization}} 
 & Optimizer & AdamW & AdamW & AdamW & AdamW & AdamW \\
 & Optimizer Betas & $(0.9, 0.95)$ & $(0.9, 0.95)$ & $(0.9, 0.95)$ & $(0.9, 0.95)$ & $(0.9, 0.95)$ \\
 & Peak LR & $2.5 \times 10^{-5}$ & $2.5 \times 10^{-5}$ & $2.5 \times 10^{-5}$ & $2.5 \times 10^{-5}$ & $2.5 \times 10^{-5}$ \\
 & Min LR & $2.5 \times 10^{-6}$ & $2.5 \times 10^{-6}$ & $2.5 \times 10^{-6}$ & $2.5 \times 10^{-6}$ & $2.5 \times 10^{-6}$ \\
 & Weight Decay & $1.0 \times 10^{-10}$ & $1.0 \times 10^{-10}$ & $1.0 \times 10^{-10}$ & $1.0 \times 10^{-10}$ & $1.0 \times 10^{-10}$ \\
 & Gradient Clipping & 1 & 1 & 1 & 1 & 1 \\
 & LR Scheduler & Cosine Decay & Cosine Decay & Cosine Decay & Cosine Decay & Cosine Decay \\
 & Warmup Steps & 1,000 & 1,000 & 1,000 & 1,000 & 1,000 \\
 & Training Steps & 30,000 & 70,000 & 100,000 & 70,000 & 30,000 \\ 
 & Batch Size & 16 & 16 & 16 & 32 & 16 \\ 
\midrule
\multirow{6}{*}{\textbf{Data}} 
 & Episodes & 1,693 & 6,000 & 25,460 & 26,152 & 700 \\
 & Frames & 273,465 & 585,000 & 864,292 & 1,067,618 & 288,594 \\
 & Action Dim & 7 & 8 & 7 & 7 & 7 / 14 \\
 & Cameras & $1\times$Wrist $+ 1\times$3rd & $1\times$Wrist $+ 1\times$3rd & $1\times$3rd & $1\times$3rd & $2\times$Wrist $+ 1\times$3rd \\
 & State Dim & 8 & 25 & -- & -- & 7 / 14 \\
 & Instruction & Natural & Natural & Natural & Templated & Natural \\ 
\bottomrule
\end{tabular}%
}
\end{table}

%% file: Proof.tex
\begin{definition}[Logarithmic Sampling Sequence]
Let $S = \{x_0, x_1, x_2, \dots\}$ be a strictly increasing integer sequence with anchor $x_0 := 0$.
For $i \ge 1$, the elements are generated by a base parameter $q_{\min} > 0$ and a growth rate $r > 1$:
\begin{equation}
    x_i = \lfloor q_{\min} \cdot r^i \rfloor .
\end{equation}
\end{definition}

\begin{definition}[Sparse Sampling Constraint]
The sequence $S$ is said to satisfy the \emph{sparse sampling constraint} if any three consecutive terms satisfy
the second-order recurrence inequality
\begin{equation}
    x_i \ge x_{i-1} + x_{i-2}, \quad \forall i \ge 2 .
\end{equation}
\end{definition}

\begin{definition}[Coincidence Count]
For a given step size $\Delta \in \mathbb{Z}^+$, define the coincidence count function
\begin{equation}
    H(\Delta)
    =
    \left|
    \left\{
    (j,i) \in \mathbb{Z}_{\ge 0}^2
    \;\middle|\;
    j>i,\;
    x_j - x_i = \Delta
    \right\}
    \right|.
\end{equation}
\end{definition}

\begin{proposition}[Bounded Coincidence Count]
Let $S$ be any sequence satisfying the sparse sampling constraint.
Then for all $\Delta > 0$,
\begin{equation}
    H(\Delta) \le 3 .
\end{equation}
Moreover, equality can only occur if the sequence attains the Fibonacci boundary
$x_i = x_{i-1} + x_{i-2}$ and $\Delta \in S$.
\end{proposition}

\begin{proof}
We consider solutions to the equation $x_j - x_i = \Delta$ with $j>i$.

Any sequence satisfying $x_i \ge x_{i-1} + x_{i-2}$ dominates the Fibonacci sequence term-wise.
Consequently, for any fixed index pair $(j,i)$, the difference $x_j - x_i$ is minimized in the
tight boundary case where equality holds, i.e.,
\begin{equation}
    x_i = x_{i-1} + x_{i-2}.
\end{equation}
Hence, the maximal possible value of $H(\Delta)$ is achieved under the Fibonacci recurrence,
which corresponds to the minimal growth rate compatible with the sparse sampling constraint.
It therefore suffices to analyze this extremal case. Without loss of generality, consider $\Delta = x_k \in S$ for some $k \ge 1$. If $\Delta \notin S$, monotonicity immediately reduces the number of admissible solutions. We enumerate all possible solutions to $x_j - x_i = x_k$:
\begin{description}[font=\bfseries, style=nextline, labelindent=2em, leftmargin=4em]

\item[Case 1 ($i=0$):]
Since $x_0 = 0$, the trivial solution $(j,i) = (k,0)$ always exists.

\item[Case 2 ($j=i+1$):]
From the Fibonacci recurrence,
\[
x_{i+1} - x_i = x_{i-1}.
\]
Setting $x_{i-1} = x_k$ yields $i = k+1$, giving the solution $(j,i) = (k+2,k+1)$.

\item[Case 3 ($j=i+2$):]
We have
\[
x_{i+2} - x_i = (x_{i+1} + x_i) - x_i = x_{i+1}.
\]
Setting $x_{i+1} = x_k$ yields $i = k-1$, giving the solution $(j,i) = (k+1,k-1)$.

\item[Case 4 ($j-i \ge 3$):]
We show that no solutions exist in this regime.
\begin{itemize}[label=\textbullet]
    \item If $j \ge k+2$, then for all admissible $i$,
    \[
    x_j - x_i \ge x_{k+2} - x_{k-1} = 2x_k > x_k .
    \]
    \item If $j = k+1$ and $i \le k-2$,
    \[
    x_{k+1} - x_{k-2} = 2x_{k-1} > x_k,
    \]
    since $x_k = x_{k-1} + x_{k-2}$.
\end{itemize}
Thus, $x_j - x_i$ strictly exceeds $x_k$, and no solutions arise.

\end{description}

Collecting all cases, there are at most three admissible solutions.
For sequences with strict inequality $x_i > x_{i-1} + x_{i-2}$,
Cases 2 and 3 are no longer attainable, and $H(\Delta)$ is further reduced.
Therefore, for all sequences satisfying the sparse sampling constraint,
\[
H(\Delta) \le 3 .
\]
\end{proof}

\paragraph{Remark (Connection to Logarithmic Sampling)}
Although the Fibonacci sequence is defined by a linear recurrence,
its closed-form expression given by Binet's formula is
\begin{equation}
    F_i = \frac{\phi^i - \psi^i}{\sqrt{5}},
\end{equation}
where $\phi = \frac{1+\sqrt{5}}{2}$ and $\psi = \frac{1-\sqrt{5}}{2}$.
Since $|\psi|<1$, the term $\psi^i$ decays exponentially, yielding the asymptotic relation
\begin{equation}
    F_i \sim \frac{1}{\sqrt{5}} \phi^i .
\end{equation}
Thus, the Fibonacci sequence can be viewed as an asymptotic logarithmic sampling sequence
with growth rate $r=\phi$.
This observation justifies treating the Fibonacci recurrence as the extremal boundary case
when analyzing coincidence bounds for logarithmic hindsight sampling sequences.

%% file: assets/HMI.tex
\begin{figure*}[htbp]
    \centering

    % --- 第一行 (Row 1) ---
    \begin{subfigure}[b]{0.32\textwidth} % 宽度调整为 0.32
        \centering
        \includegraphics[width=\linewidth]{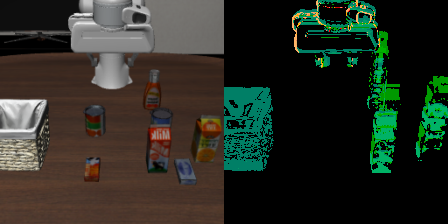}
        \caption{}
        \label{fig:sub1}
    \end{subfigure}
    \hfill % 弹性空格
    \begin{subfigure}[b]{0.32\textwidth}
        \centering
        \includegraphics[width=\linewidth]{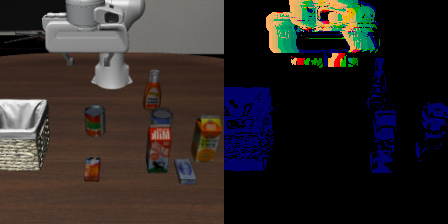}
        \caption{}
        \label{fig:sub2}
    \end{subfigure}
    \hfill
    \begin{subfigure}[b]{0.32\textwidth}
        \centering
        \includegraphics[width=\linewidth]{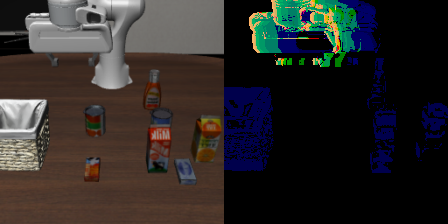}
        \caption{}
        \label{fig:sub3}
    \end{subfigure}

    \par\bigskip % 强制换行并增加一点垂直距离
    % \vspace{1em} % 或者用这个精确控制行间距

    % --- 第二行 (Row 2) ---
    \begin{subfigure}[b]{0.32\textwidth}
        \centering
        \includegraphics[width=\linewidth]{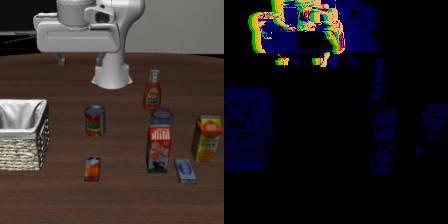}
        \caption{}
        \label{fig:sub4}
    \end{subfigure}
    \hfill
    \begin{subfigure}[b]{0.32\textwidth}
        \centering
        \includegraphics[width=\linewidth]{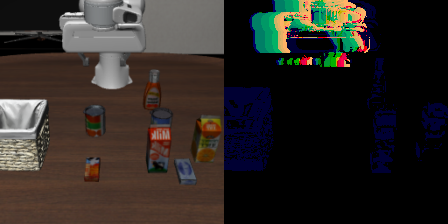}
        \caption{}
        \label{fig:sub5}
    \end{subfigure}
    \hfill
    \begin{subfigure}[b]{0.32\textwidth}
        \centering
        \includegraphics[width=\linewidth]{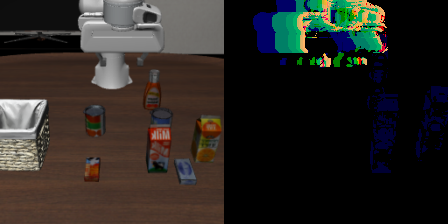}
        \caption{} % 修正了这里原本重复的 caption
        \label{fig:sub6}
    \end{subfigure}

    \caption{Visualization of visual inputs and corresponding image hindsight in the LIBERO simulation environment (10Hz). Each subfigure (a--f) displays the raw RGB camera observation on the left and the processed output from the channel-wise temporal encoding (CTE) module on the right. The Image Hindsight captures the temporal progression of the task within the latent representation.}
    \label{fig:six_images}
\end{figure*}

\begin{figure*}[!ht]
    \centering

    % --- 第一行 (Row 1) ---
    \begin{subfigure}[b]{0.32\textwidth} % 宽度调整为 0.32
        \centering
        \includegraphics[width=\linewidth]{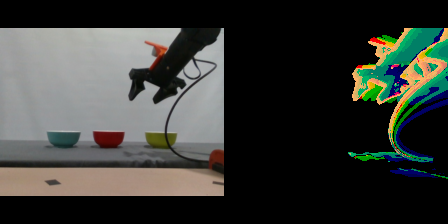}
        \caption{}
        \label{fig:sub1}
    \end{subfigure}
    \hfill % 弹性空格
    \begin{subfigure}[b]{0.32\textwidth}
        \centering
        \includegraphics[width=\linewidth]{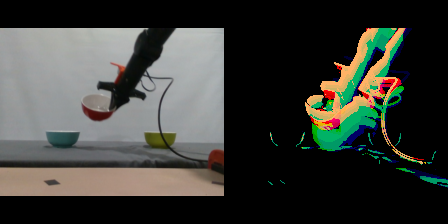}
        \caption{}
        \label{fig:sub2}
    \end{subfigure}
    \hfill
    \begin{subfigure}[b]{0.32\textwidth}
        \centering
        \includegraphics[width=\linewidth]{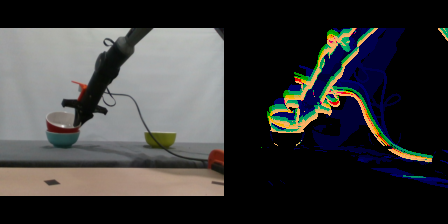}
        \caption{}
        \label{fig:sub3}
    \end{subfigure}

    \par\bigskip % 强制换行并增加一点垂直距离
    % \vspace{1em} % 或者用这个精确控制行间距

    % --- 第二行 (Row 2) ---
    \begin{subfigure}[b]{0.32\textwidth}
        \centering
        \includegraphics[width=\linewidth]{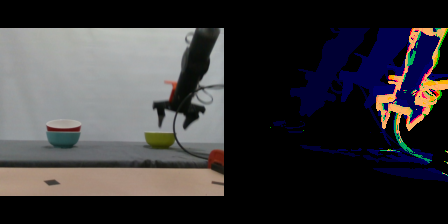}
        \caption{}
        \label{fig:sub4}
    \end{subfigure}
    \hfill
    \begin{subfigure}[b]{0.32\textwidth}
        \centering
        \includegraphics[width=\linewidth]{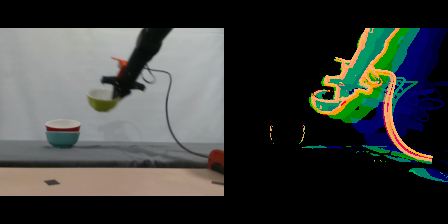}
        \caption{}
        \label{fig:sub5}
    \end{subfigure}
    \hfill
    \begin{subfigure}[b]{0.32\textwidth}
        \centering
        \includegraphics[width=\linewidth]{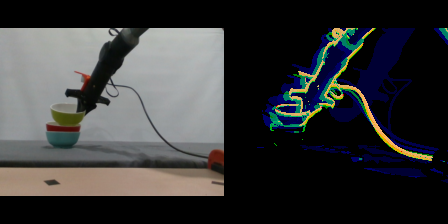}
        \caption{} % 修正了这里原本重复的 caption
        \label{fig:sub6}
    \end{subfigure}

    \caption{Visualization of visual inputs and image hindsight in real-world Piper robotic arm experiments (30Hz). Similar to the simulation, the left panels show raw physical observations, while the right panels display the corresponding processed image hindsight.}
    \label{fig:six_images}
\end{figure*}